\documentclass{article}

 \usepackage[preprint]{neurips_2026}

\usepackage[utf8]{inputenc} % allow utf-8 input
\usepackage[T1]{fontenc}    % use 8-bit T1 fonts
\usepackage{hyperref}       % hyperlinks
\usepackage{url}            % simple URL typesetting
\usepackage{booktabs}       % professional-quality tables
\usepackage{amsmath}
\usepackage{amsfonts}       % blackboard math symbols
\usepackage{amssymb}
\usepackage{amsthm}
\usepackage{nicefrac}       % compact symbols for 1/2, etc.
\usepackage{microtype}      % microtypography
\usepackage{xcolor}         % colors
\usepackage{graphicx}
\usepackage{multirow}

\newtheorem{proposition}{Proposition}

\title{Stochastic Flow Map for Count Data}

\author{%
  Ganchao Wei\thanks{Alternative email: \texttt{weiganchao@gmail.com}} \\
  Flatiron Institute\\
  New York, NY 10010, USA\\
  \texttt{gwei@flatironinstitute.org}
}

\begin{document}

\maketitle

\begin{abstract}
High-dimensional count data are common in scientific applications, but most diffusion and flow models are designed for continuous or categorical data, and generation often requires many sequential model evaluations. Here, we propose Count Flow Map, a generative model that learns finite-time transitions directly in count space for one- or few-step generation and transport between count distributions. Our model directly learns stochastic transitions over finite time intervals, using Poisson births and Binomial deaths to preserve nonnegative integer counts without a predefined maximum. These transition models are trained to match the underlying local birth--death dynamics and to maintain consistency across step sizes. We characterize the connection between local dynamics and finite-time transition consistency and derive a bound on the generation error.  After validating Count Flow Map in several simulations, we apply it to single-cell drug perturbation prediction and neural population forecasting, where it captures perturbation responses and supports forecasts of high-activity events with only one or a few model evaluations. Together, these experiments demonstrate that Count Flow Map enables high-quality generation directly in count space across inference budgets, from one-step to few-step generation, using a single trained model.
\end{abstract}

\section{Introduction}
\label{sec:introduction}

High-dimensional count data arise in many scientific applications, including
single-cell RNA sequencing and neural spike trains. Their discreteness,
sparsity, and complex dependence structure make flexible joint modeling
challenging, particularly when the count range is large. Recent advances in
diffusion models and flow matching provide powerful tools for generative
modeling \citep{ho2020denoising-diffusion,lipman2023flow-matching,albergo2025stochastic-interpolants},
but their application to count data has two challenges.
First, most existing methods are designed for continuous or categorical
data. Applying continuous models to counts requires a continuous
representation of the data, while categorical models treat each possible
count as a separate category
\citep{austin2021d3pm,campbell2022continuous-time-discrete,gat2024discrete-flow-matching}.
This categorical representation typically requires specifying a maximum
count, and the model's output dimension grows with the supported count
range. Second, generation typically requires many sequential evaluations
of the learned dynamics, making sampling costly even when training is
efficient.

Count-specific generative models have begun to address the representation
problem, including count-FM, Count Bridges, and CountsDiff
\citep{wei2026count-flow-matching,fishman2026count-bridges,soatto2026countsdiff}.
In particular, count-FM models transport between arbitrary source and target
count distributions directly on $\mathbb{N}_0^d$ through local birth--death
dynamics. Its conditional binomial bridge enables simulation-free training
of the marginal transition rates, while its rate parameterization avoids
enumerating possible count levels. Count-FM therefore provides a direct and
parameter-efficient formulation for count-valued transport. However, it
learns infinitesimal transition rates, and generation still requires
repeatedly simulating short-time transitions of the learned jump process.
The remaining challenge is to retain this count-valued formulation while
reducing the number of model evaluations needed for generation.
Flow Maps provide a way to address this sampling cost by learning
finite-time evolution directly
\citep{boffi2025flow-map-matching,boffi2025flow-map-self-distillation}.
For continuous dynamics, a two-time flow map predicts the state reached
over a specified interval, allowing one- or few-step generation from the
same trained model. This perspective connects to consistency models, consistency trajectory models, progressive distillation, and shortcut models \citep{song2023consistency-models,kim2024consistency-trajectory,salimans2022progressive-distillation,frans2025shortcut-models}.
Recent Discrete Flow Maps and Categorical Flow Maps extend accelerated
generation to categorical data through probability-simplex representations
\citep{potaptchik2026discrete-flow-maps,roos2026categorical-flow-maps}.
These developments motivate extending count-FM from learning local
birth--death rates to learning the corresponding finite-time transitions, while preserving its unbounded count space.

We propose \emph{Count Flow Map}, which extends count-FM from local birth--death dynamics to finite-time transitions, enabling one- or few-step generation and transport directly in count space. Count Flow Map models stochastic transitions between pairs of time points using Poisson births and Binomial deaths, ensuring that samples remain nonnegative integers without imposing a predefined maximum count. The transition models recover count-FM's local unit births and deaths in the infinitesimal-time limit, reflecting the ordered structure of count space. They are trained to remain consistent when longer transitions are composed from shorter ones.
This allows Count Flow Map to support different inference budgets without retraining.

We evaluate Count Flow Map in 2-D and 32-D simulations and in two biological applications, single-cell drug perturbation prediction using Tahoe-100M \citep{zhang2025tahoe100m} and neural population forecasting using the recordings of \citet{steinmetz2019distributed}. Overall, these experiments show that Count Flow Map achieves strong distributional quality at low sampling cost, requiring only one or a few model evaluations.
Our main contributions are as follows:
\begin{itemize}
\item We formulate Count Flow Map using stochastic transition kernels on an unbounded count space and characterize their connection to local birth--death dynamics through the forward, backward, and compositional consistency conditions.
\item We develop a tractable transition model with Poisson births and Binomial deaths, with finite-time corrections that preserve the local birth--death dynamics. The training objective combines count-FM rate matching with transition consistency across time intervals.
\item We derive a terminal-error bound that separates endpoint truncation, generator estimation, and finite-time approximation, and quantify how local approximation errors and transition inconsistencies affect finite-time accuracy.
\end{itemize}

\section{Method}
\label{sec:method}

In this section, we formulate Count Flow Map, introduce a finite-time transition model for count data, derive the training objective, and describe one- and few-step generation. Proofs and additional derivations are given in Appendix~\ref{app:method-theory}. A Python implementation of Count Flow Map is available at \url{https://github.com/weigcdsb/count-flow-map}.

\subsection{Count Flow Map}
\label{subsec:count-flow-map-definition}

Let $\mathcal{X}=\mathbb{N}_0^d$ denote the space of $d$-dimensional nonnegative count vectors, and let $p_0$ and $p_1$ be arbitrary source and target distributions on $\mathcal{X}$. We consider a probability path $(p_t)_{t\in[0,1]}$ connecting them, with endpoints $p_0$ and $p_1$.
Here, we use stochastic transitions because count distributions are discrete and transport may require splitting atomic probability mass, while a deterministic map assigns only one output to each input count vector.
For $0\le s\le t\le1$, a \emph{Count Flow Map} is a stochastic map represented by a Markov transition kernel
\begin{equation}
K_{s,t}(y\mid x)=\mathbb{P}(X_t=y\mid X_s=x),
\qquad x,y\in\mathcal{X},
\label{eq:count-flow-map}
\end{equation}
that transports the path marginals, $p_t(y)=\sum_{x\in\mathcal{X}}p_s(x)K_{s,t}(y\mid x)$. Thus, for a fixed state $x$ at time $s$, the finite-time map returns the distribution of the state at time $t$. For a function $f:\mathcal{X}\to\mathbb{R}$, we write $(K_{s,t}f)(x)=\sum_yK_{s,t}(y\mid x)f(y)$ and let $I$ denote the identity kernel. The infinitesimal generator of the transition family is
\begin{equation}
Q_t=\lim_{h\downarrow0}\frac{K_{t,t+h}-I}{h}.
\label{eq:count-flow-map-generator}
\end{equation}
We refer to $Q_t$ as the diagonal generator, as it is determined by $K_{s,t}$ near the diagonal $s=t$. For each coordinate $i\in\{1,\ldots,d\}$, local unit births and deaths change $x_i$ to $x_i+1$ and $x_i-1$, respectively. Let $e_i$ denote the $i$th standard basis vector. The resulting birth--death generator is
\begin{equation}
(Q_tf)(x)
=
\sum_{i=1}^d\lambda_{t,i}(x)\bigl[f(x+e_i)-f(x)\bigr]
+
\sum_{i=1}^d\mu_{t,i}(x)\bigl[f(x-e_i)-f(x)\bigr],
\label{eq:birth-death-generator}
\end{equation}
where $\lambda_{t,i}(x)\ge0$ and $\mu_{t,i}(x)\ge0$ are the coordinate-wise birth and death rates, with $\mu_{t,i}(x)=0$ whenever $x_i=0$.

For two kernels $K$ and $M$, define their composition by $(KM)(y\mid x)=\sum_z K(z\mid x)M(y\mid z)$. The consistency conditions for deterministic Flow Maps \citep{boffi2025flow-map-self-distillation} have the following stochastic counterparts:
\begin{equation}
\begin{aligned}
\partial_tK_{s,t}&=K_{s,t}Q_t,
&&\text{Lagrangian (forward) consistency},\\
\partial_sK_{s,t}&=-Q_sK_{s,t},
&&\text{Eulerian (backward) consistency},\\
K_{s,t}&=K_{s,u}K_{u,t},\qquad s\le u\le t,
&&\text{compositional consistency}.
\end{aligned}
\label{eq:count-flow-map-consistency}
\end{equation}
The first two are the forward and backward Kolmogorov equations, and the third expresses compositional consistency through the Chapman--Kolmogorov identity
\citep{norris1997markov-chains,feinberg2014kolmogorov}. Appendix~\ref{app:flow-map-consistency-proof} derives all three identities for the exact transition family and proves the converse characterization used below.

Let $P^Q_{s,t}(y\mid x)$ denote the transition kernel of the continuous-time Markov process generated by $Q_t$, i.e., the probability of reaching $y$ at time $t$ given $x$ at time $s$.

\begin{proposition}[Characterization of a Count Flow Map]
\label{prop:count-flow-map-characterization}
Suppose that $Q_t$ generates a unique nonexplosive transition family $P^Q_{s,t}$. Under the regularity conditions in Appendix~\ref{app:flow-map-consistency-proof}, a transition family $K_{s,t}$ with $K_{t,t}=I$ equals $P^Q_{s,t}$ if it satisfies either the forward equation, the backward equation, or compositional consistency with infinitesimal generator $Q_t$.
\end{proposition}

Thus, the diagonal generator $Q_t$ specifies the local dynamics, while any one of the three consistency conditions identifies the corresponding finite-time transition family. For training, we use compositional consistency $K_{s,t}=K_{s,u}K_{u,t}$ because it can be enforced directly through kernel composition.

\subsection{Transition kernel for count data}
\label{subsec:count-transition-kernel}

Count-FM learns local birth and death rates, so generation requires repeated simulation of the resulting dynamics. To enable direct sampling over a finite interval from $s$ to $t$, we introduce a tractable parametric model $K_{\theta,s,t}$ for the Count Flow Map $K_{s,t}$. Its diagonal generator $Q_t^\theta$ has the same birth--death form as $Q_t$ in Equation~\eqref{eq:birth-death-generator}, with birth and death rates $\lambda_{\theta,i}(x,t)$ and $\mu_{\theta,i}(x,t)$, respectively. We parameterize the death rate as $\mu_{\theta,i}(x,t)=x_i\beta_{\theta,i}(x,t)$, where $\beta_{\theta,i}(x,t)\ge0$ is the per-count death rate.
For a finite interval from $s$ to $t$, let $\Delta=t-s$. Given the
current state $x$ at time $s$, we model coordinate $i$ by removing a
random number of existing counts and adding a random number of new counts,
\begin{equation}
Y_i=x_i-D_i+B_i,
\qquad
B_i\sim\operatorname{Poisson}(a_i),
\qquad
D_i\sim\operatorname{Binomial}(x_i,q_i).
\label{eq:poisson-binomial-transition}
\end{equation}
Because $0\le D_i\le x_i$ and $B_i\ge0$, this transition remains in $\mathbb{N}_0$ without imposing an upper bound on the count. This construction is related to classical integer-valued autoregressive models, which combine binomial thinning of existing counts with new arrivals \citep{alosh1987inar}, providing a natural count-valued form for modeling finite-time births and deaths.

In Equation~\eqref{eq:poisson-binomial-transition}, $a_i(x,s,t)\ge0$ is the Poisson mean for the number of births, while $q_i(x,s,t)\in[0,1]$ is the Binomial probability that each count present at time $s$ is removed by time $t$. These finite-time parameters are constructed so that, as $t\downarrow s$, $a_i$ and $q_i$ match the local birth rate $\lambda_{\theta,i}(x,s)$ and per-count death rate $\beta_{\theta,i}(x,s)$ to first order, while finite-time corrections provide flexibility to learn departures from this local approximation over longer intervals. 
To capture dependence across coordinates, we use a finite mixture of these Poisson--Binomial kernels, with the mixture component shared across coordinates.
The precise parameterization and the resulting probability mass function are given in Appendix~\ref{app:count-kernel-likelihood}.

\begin{proposition}[Diagonal generator consistency]
\label{prop:diagonal-generator-consistency}
Suppose that the finite-time birth and death parameters satisfy
$a_i(x,t,t+h)=h\lambda_{\theta,i}(x,t)+O(h^2)$ and
$q_i(x,t,t+h)=h\beta_{\theta,i}(x,t)+O(h^2)$ as $h\downarrow0$.
Then $K_{\theta,t,t}=I$ and
\[
K_{\theta,t,t+h}(x+e_i\mid x)
=h\lambda_{\theta,i}(x,t)+O(h^2),
\qquad
K_{\theta,t,t+h}(x-e_i\mid x)
=h\mu_{\theta,i}(x,t)+O(h^2).
\]
All transitions involving two or more unit events have probability $O(h^2)$. Hence
\[
\lim_{h\downarrow0}\frac{K_{\theta,t,t+h}-I}{h}=Q_t^\theta,
\]
where $Q_t^\theta$ is the birth--death generator introduced above.
\end{proposition}

The proof is provided in Appendix~\ref{app:diagonal-generator-proof}. Thus, the finite-time kernel $K_{\theta,s,t}$ retains $Q_t^\theta$ as its infinitesimal generator, preserving the intended local birth--death dynamics.

\subsection{Training objective}
\label{subsec:count-flow-map-training}

We train the diagonal generator and the off-diagonal finite-time transitions jointly. The diagonal objective coincides with count-FM \citep{wei2026count-flow-matching}, and the off-diagonal objective enforces compositional consistency.

\paragraph{Diagonal rate matching.}
Let $\pi$ be a coupling of $p_0$ and $p_1$, and draw $(X_0,X_1)\sim\pi$. Conditional on $(X_0=x_0,X_1=x_1)$, we use the signed-binomial bridge from count-FM. For each coordinate $i$ and time $t\in[0,1]$,
\begin{equation}
X_t^{(i)}
=x_{0,i}+\operatorname{sgn}(x_{1,i}-x_{0,i})B_t^{(i)},
\qquad
B_t^{(i)}\sim\operatorname{Binomial}\!\left(|x_{1,i}-x_{0,i}|,t\right),
\label{eq:signed-binomial-bridge}
\end{equation}
independently across coordinates given the endpoints. 
As shown in count-FM \citep{wei2026count-flow-matching}, substituting this bridge into the one-dimensional mass-preservation equation for local birth--death dynamics gives the conditional target rates
\begin{equation}
\bar\lambda_{t,i}
=\frac{(x_{1,i}-X_t^{(i)})_+}{1-t},
\qquad
\bar\mu_{t,i}
=\frac{(X_t^{(i)}-x_{1,i})_+}{1-t},
\label{eq:conditional-bridge-rates}
\end{equation}
where $(\cdot)_+=\max(\cdot,0)$. These rates are singular at $t=1$, therefore we set $\tau=1-\varepsilon$ for a small $\varepsilon>0$ and train on $t\in[0,\tau]$. The diagonal loss is
\begin{equation}
\mathcal{L}_{\mathrm{diag}}(\theta)
=
\mathbb{E}_{\substack{
(X_0,X_1)\sim\pi,\\
t\sim\operatorname{Unif}[0,\tau],\\
X_t\sim p_t(\cdot\mid X_0,X_1)
}}
\left[
\sum_{i=1}^d
\ell\!\left(\bar\lambda_{t,i},\lambda_{\theta,i}(X_t,t)\right)
+
\sum_{i=1}^d
\ell\!\left(\bar\mu_{t,i},\mu_{\theta,i}(X_t,t)\right)
\right].
\label{eq:diagonal-rate-loss}
\end{equation}
where $\ell(a,b)=b-a\log b$. A small numerical floor inside the logarithm is used in implementation.

The corresponding marginal rates are
$\lambda_i^\star(x,t)=\mathbb{E}[\bar\lambda_{t,i}\mid X_t=x]$
and
$\mu_i^\star(x,t)=\mathbb{E}[\bar\mu_{t,i}\mid X_t=x]$.
Although the conditional bridge is independent across coordinates given
$(X_0,X_1)$, the marginal rates depend on the full state $x$ and can
therefore capture cross-coordinate dependence. Generator matching shows
that these rates generate the marginal probability path
\citep{holderrieth2025generator-matching}. Moreover, the conditional
objective in Equation~\eqref{eq:diagonal-rate-loss} has the same population
optimum as directly matching $\lambda_i^\star$ and $\mu_i^\star$. Thus the
diagonal part of Count Flow Map training is exactly count-FM and identifies
the infinitesimal generator $Q_t^\star$ of the prescribed count path. The derivation is provided in Appendix~\ref{app:count-fm-diagonal-matching}.

\paragraph{Off-diagonal compositional consistency.}
For any $0\le s<u<t\le\tau$, compositional consistency requires the direct transition over $[s,t]$ to agree with the composition of transitions over $[s,u]$ and $[u,t]$. To construct the self-distillation target, let $\widetilde K_\theta$ denote $K_\theta$ with gradients stopped. Starting from $X_s\sim p_s$, draw $Z\sim\widetilde K_{\theta,s,u}(\cdot\mid X_s)$ and then $Y\sim\widetilde K_{\theta,u,t}(\cdot\mid Z)$. Equivalently, $Y$ is drawn from the composed target kernel
$R_{\theta;s,u,t}=\widetilde K_{\theta,s,u}\widetilde K_{\theta,u,t}$. We fit the direct transition with
\begin{equation}
\mathcal{L}_{\mathrm{comp}}(\theta)
=
-\mathbb{E}_{\substack{
X_s\sim p_s,\\
Y\sim R_{\theta;s,u,t}(\cdot\mid X_s)
}}
\left[
\log K_{\theta,s,t}(Y\mid X_s)
\right].
\label{eq:chapman-kolmogorov-loss}
\end{equation}
Because gradients are stopped through $R_{\theta;s,u,t}$, the compositional consistency loss differs from the expected Kullback-Leibler (KL) divergence $\mathbb{E}_{X_s\sim p_s}[D_{\mathrm{KL}}(R_{\theta;s,u,t}(\cdot\mid X_s)\Vert K_{\theta,s,t}(\cdot\mid X_s))]$ only by the target entropy, so the two objectives have identical gradients when the
composed target is held fixed.
Since $\widetilde K_\theta$ and $K_\theta$ have identical values, zero KL discrepancy implies $K_{\theta,s,t}=K_{\theta,s,u}K_{\theta,u,t}$.

The complete objective is
\begin{equation}
\mathcal{L}(\theta)
=\mathcal{L}_{\mathrm{diag}}(\theta)
+\alpha\mathcal{L}_{\mathrm{comp}}(\theta),
\label{eq:count-flow-map-loss}
\end{equation}
where $\alpha\ge0$ controls the off-diagonal consistency term. In this paper, we set
$\alpha=1$ in all experiments. The diagonal loss identifies the local generator, while the compositional consistency loss penalizes deviations between direct and composed finite-time transitions.
Here, we choose $u=(s+t)/2$ to balance the two subintervals.
In practice, we can further extend two-step composition to multiple
transitions over equally spaced subintervals, to encourage more accurate
mappings over longer time intervals. Under uniform short-time accuracy,
exact midpoint consistency identifies the transition family generated
by $Q_t^\theta$ (Appendix~\ref{app:midpoint-identification}).

\paragraph{Terminal error decomposition.}
Let $P^\theta_{s,t}$ denote the exact transition family generated by the learned diagonal generator $Q_t^\theta$. For an inference partition $\Pi=\{0=t_0<\cdots<t_L=\tau\}$, write $K_\theta^\Pi=K_{\theta,t_0,t_1}\cdots K_{\theta,t_{L-1},t_L}$. Appendix~\ref{app:terminal-error-analysis} shows
\begin{equation}
\begin{aligned}
D_{\mathrm{TV}}(p_1,p_0K_\theta^\Pi)
\le{}&
\underbrace{\varepsilon\,\mathbb{E}_{\pi}\|X_1-X_0\|_1}_{\text{endpoint truncation}}
+
\underbrace{\sqrt{\frac{\tau}{2}\left[\mathcal{L}_{\mathrm{diag}}(\theta)-\mathcal{L}_{\mathrm{diag}}^\star\right]}}_{\text{generator estimation}}
+
\underbrace{\mathcal{E}_{\mathrm{map}}(\Pi)}_{\text{finite-time map}},
\end{aligned}
\label{eq:terminal-error-bound}
\end{equation}
where $D_{\mathrm{TV}}$ denotes total variation (TV) distance, $\mathcal{L}_{\mathrm{diag}}^\star$ is the population diagonal loss at the exact marginal rates and $\mathcal{E}_{\mathrm{map}}(\Pi)=D_{\mathrm{TV}}(p_0P^\theta_{0,\tau},p_0K_\theta^\Pi)$.

The first two terms follow count-FM \citep{wei2026count-flow-matching}. The first term favors smaller truncation $\varepsilon$ and closer endpoint
couplings (e.g., $L_1$-optimal transport), while the second favors
more accurate diagonal rate matching.
Under uniform local accuracy and finite $\eta_{\mathrm{comp}}$,
the map error (the third term) satisfies
$\mathcal E_{\mathrm{map}}(\Pi)\le 2\eta_{\mathrm{comp}}\tau^2/L$
for $L$ equally spaced inference steps
(Appendix~\ref{app:finite-time-map-error}), where $\eta_{\mathrm{comp}}$ denotes the normalized worst-case
TV discrepancy between direct and two-step midpoint transitions.
Thus, better midpoint consistency is particularly relevant
when the number of function evaluations (NFE) is small, while increasing $L$ tightens this upper bound.
% Under the additional uniform
% accuracy and normalized midpoint residual assumptions in
% Appendix~\ref{app:finite-time-map-error}, the map error (the third term) satisfies
% $\mathcal E_{\mathrm{map}}(\Pi)\le 2\eta_{\mathrm{comp}}\tau^2/L$
% for $L$ equally spaced inference steps. Thus, better midpoint
% consistency is particularly relevant at low NFE, while increasing $L$
% tightens this upper bound.

\subsection{Sample generation}
\label{subsec:count-flow-map-generation}

After training, the Count Flow Map supports different inference budgets
without retraining. Given a partition
$0=t_0<t_1<\cdots<t_L=\tau$, we sample $X_{t_0}\sim p_0$ and sequentially
apply $K_{\theta,t_k,t_{k+1}}$ for $k=0,\ldots,L-1$.
Thus, $L=1$ gives direct one-step generation with $K_{\theta,0,\tau}$,
while $L>1$ performs multi-step generation over shorter time intervals.
Under exact compositional consistency, these choices give the same
finite-time transition. For conditional generation, covariates $c$ are
provided as additional model inputs, giving
$K_{\theta,s,t}(\cdot\mid x,c)$, with the same sampling procedure.
All experiments were conducted on a single NVIDIA GeForce RTX 2080 Ti GPU.

\section{Simulation}
\label{sec:simulation}

Here, we consider two simulation settings to evaluate endpoint quality and inference
efficiency across different numbers of function evaluations (NFE). We begin
with a 2-D example, which is simple and easy to visualize. We then consider a
32-dimensional high-count setting to assess performance on a more challenging
count distribution.

We compare Count Flow Map with Count-FM using its original unit-jump sampler
\citep{wei2026count-flow-matching} and a binomial $\tau$-leap variant motivated
by binomial leap methods \citep{tian2004binomial-leap}, CountsDiff
\citep{soatto2026countsdiff}, D3PM with ordinal Gaussian transitions
\citep{austin2021d3pm}, and Discrete Flow Maps
\citep{potaptchik2026discrete-flow-maps}. The two Count-FM samplers use the same trained rate
network and differ only at inference. All learned models are trained for
10,000 updates over five random seeds, using MLPs with hidden width 256 and
depth 4 for a fair comparison. We evaluate NFE in
$\{1,2,4,16,64,128,256\}$. Additional implementation and evaluation details
are given in Appendix~\ref{app:simulation-details}.

\paragraph{2-D simulation.}
The target is an equal-weight mixture of two Gamma--Poisson ($\mathrm{GP}$) distributions,
$\mathrm{GP}((60,5),(160,80))$ and
$\mathrm{GP}((60,40),(160,140))$, where $\mathrm{GP}(m,\alpha)$ denotes a
Gamma--Poisson distribution with mean $m$ and concentration $\alpha$.
For Count Flow Map and Count-FM, the source is
$p_0=\mathrm{Unif}\{0,\ldots,72\}^2$, and the signed-binomial bridge is
truncated at $\tau=0.995$. The remaining baselines follow their original
reference distributions, with details given in
Appendix~\ref{app:simulation-details}. We report total variation (TV) distance and empirical 2-Wasserstein distance ($W_2$). Figure~\ref{fig:sim-2d-efficiency} shows the quality--efficiency curves, with detailed values at selected operating points
reported in Appendix~\ref{app:simulation-2d-results}.

Count Flow Map achieves the lowest mean one-step TV and $W_2$ (NFE=1),
with TV $0.131\pm0.004$ and $W_2=2.47\pm0.34$.
Across the evaluated inference budgets, its mean TV decreases from $0.131$
at NFE=1 to $0.103$ at NFE=256, while mean $W_2$ ranges from $2.25$ to
$2.77$. At NFE=4, Count Flow Map also achieves the lowest mean TV,
whereas Discrete Flow Maps has a slightly lower mean $W_2$.
Count-FM, CountsDiff, and D3PM require more inference steps to obtain
comparable endpoint quality. These results demonstrate strong one-step
generation and consistently low endpoint error across inference budgets.

% Count Flow Map achieves the best one-step endpoint quality (NFE=1), with
% TV $0.153\pm0.005$ and $W_2=2.50\pm0.34$. Its performance is also stable
% across inference budgets, with TV between $0.126$ and $0.153$ and $W_2$
% between $2.50$ and $2.74$ from NFE=1 to 256. In contrast, Count-FM,
% CountsDiff, and D3PM require substantially more inference steps to reach
% comparable quality, while Discrete Flow Maps is competitive at low NFE,
% particularly at NFE=4. These results show that Count Flow Map retains high
% endpoint quality with one or a few evaluations and is robust to the number of
% inference steps. This stability is consistent with the learned
% compositional consistency constraint.

Endpoint samples and generation paths are shown in
Appendix~\ref{app:simulation-2d-visuals}. The endpoint samples again show that
Count Flow Map captures the two target modes at NFE=1, while Count-FM,
CountsDiff, and D3PM require more steps to approach the target. The intermediate
paths further show that Count Flow Map, similar to Count-FM, evolves smoothly
through count space. In contrast, categorical-state methods exhibit more abrupt
transitions, reflecting the mismatch between categorical representations and
the ordered geometry of count data.

\begin{figure}[!ht]
    \centering
    \includegraphics[width=\linewidth]{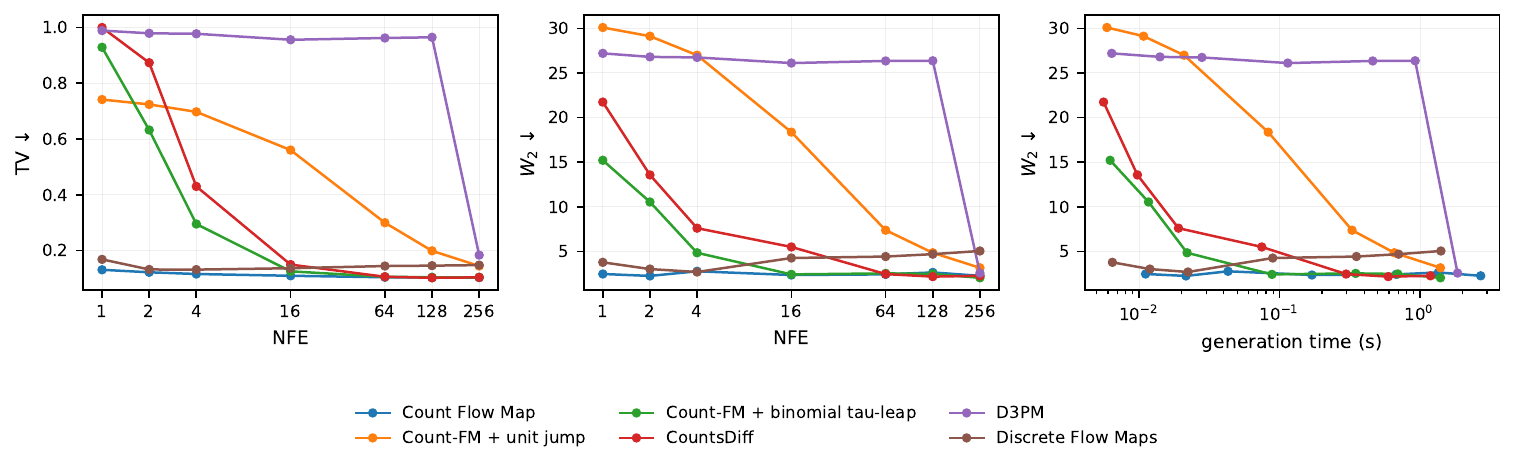}
    \caption{2-D quality--efficiency trade-offs over the full NFE grid. Left and center show endpoint TV and $W_2$ versus NFE. Right shows $W_2$ versus measured generation time for 30,000 samples. Points show means over five seeds. Count Flow Map achieves the lowest mean TV and $W_2$ at NFE=1. Its TV decreases with additional evaluations, while its $W_2$ remains low across the grid.}
    \label{fig:sim-2d-efficiency}
\end{figure}

\paragraph{32-D simulation.}
We next consider a 32-dimensional high-count setting with a three-component
Gamma--Poisson factor target and four latent factors per component. For Count Flow Map and Count-FM, the source is
$p_0=\mathrm{Unif}\{0,\ldots,50\}^{32}$.
We evaluate quality using sliced 2-Wasserstein distance
($\mathrm{SW}_2$) and squared maximum mean discrepancy with a
Gaussian radial basis function kernel ($\mathrm{MMD}^2_{\mathrm{RBF}}$)
\citep{bonneel2015sliced-wasserstein,gretton2012kernel-test}.
These metrics can be estimated directly from samples without
constructing high-dimensional count histograms for TV estimation
or solving the multivariate matching problem for empirical $W_2$. At NFE=1, Count Flow Map achieves
$\mathrm{SW}_2=0.533\pm0.132$, compared with $4.520\pm0.156$ for the
next-best method, Count-FM with binomial $\tau$-leaping.
It achieves the lowest mean $\mathrm{SW}_2$ and
$\mathrm{MMD}^2_{\mathrm{RBF}}$ at every evaluated NFE.
These results demonstrate strong low-NFE performance in this
high-dimensional, high-count setting. Full results are given in
Appendix~\ref{app:simulation-32d}.

% We next consider a 32-dimensional high-count setting with a three-component
% Gamma--Poisson factor target and four latent factors per component. For Count
% Flow Map and Count-FM, the source is
% $p_0=\mathrm{Unif}\{0,\ldots,50\}^{32}$. The same pattern persists in this
% more challenging setting. At NFE=1, Count Flow Map achieves
% $\mathrm{SW}_2=0.556\pm0.134$, compared with $4.511\pm0.148$ for the
% next-best method, and it maintains the lowest $\mathrm{SW}_2$ across all
% evaluated NFEs. It also achieves the lowest
% $\mathrm{MMD}^2_{\mathrm{RBF}}$ through NFE=128. Overall, the 32-D results show the strong low-NFE
% performance of Count Flow Map in a more challenging high-dimensional,
% high-count setting. Full results are given in
% Appendix~\ref{app:simulation-32d}.

% \section{Application to Single-Cell Drug Perturbation Prediction}
% \label{sec:application-single-cell}

\section{Applications}
\label{sec:applications}

\subsection{Single-Cell Drug Perturbation Prediction}
\label{sec:application-single-cell}

We evaluate conditional transport from control count distributions
$p_0(\cdot\mid c)$ to treated-cell count distributions
$p_1(\cdot\mid c)$ in Tahoe-100M \citep{zhang2025tahoe100m},
where $c$ specifies cell line, drug, and dose. Source cells are
DMSO controls matched by cell line and experimental plate.
The panel contains 70,064 treated cells, five cell lines, eight
drugs, three doses, and 2,000 training-selected highly variable
genes. We split the 120 conditions into 78 training, 18 validation,
and 24 test conditions, withholding all treated cells from each
held-out condition. Every held-out combination of cell line and
drug is observed at another dose during training. Full split
and preprocessing details are in Appendix~\ref{app:scrna-details}.

% We evaluate held-out drug responses in Tahoe-100M
% \citep{zhang2025tahoe100m}, generating gene-expression counts
% conditioned on cell line, drug, and dose. The panel contains
% 70,064 treated cells, five cell lines, eight drugs, three doses,
% and 2,000 training-selected highly variable genes. We split
% the 120 conditions into 78 training, 18 validation, and 24 test
% conditions, withholding all treated cells from each held-out
% condition. Every held-out cell-line--drug pair is observed at
% another dose during training. For each condition $c$ specifying cell line, drug, and dose,
% let $p_0(\cdot\mid c)$ denote the source count distribution of
% cell-line- and plate-matched DMSO controls and $p_1(\cdot\mid c)$
% the  target treated-cell count distribution. Our goal here is to learn
% conditional transport $p_0(\cdot\mid c)$ to $ p_1(\cdot\mid c)$. Full split and preprocessing details
% are in Appendix~\ref{app:scrna-details}.

We compare Count Flow Map with Count-FM's unit-jump and binomial
$\tau$-leap samplers
\citep{wei2026count-flow-matching,tian2004binomial-leap},
a conditional NB-VAE \citep{sohn2015conditional-vae,lopez2018scvi},
CPA \citep{lotfollahi2023cpa}, scGen \citep{lotfollahi2019scgen},
scVIDR \citep{kana2023scvidr}, Sinkhorn OT with dose interpolation
\citep{cuturi2013sinkhorn}, and linear dose response.
Both count-flow models use either independent or OT training couplings between matched DMSO and treated cells, following OT-CFM and count-FM \citep{tong2024ot-cfm,wei2026count-flow-matching}. For OT coupling, treated endpoints are permuted within each perturbation condition and experimental plate to minimize the total raw-count $L_1$ distance.
We report Count Flow Map at fixed 1 and 16 NFE, while Count-FM's
NFE is selected separately on validation for each sampler,
coupling, and seed. Neural checkpoints are selected by validation
sliced $W_2$.

\paragraph{Held-out perturbation prediction.}
Table~\ref{tab:scrna-endpoint-app} compares generated and held-out treated-cell distributions using sliced $W_2$ \citep{bonneel2015sliced-wasserstein}. We compute this distance after PCA projection, with PCA fitted only on the training data. We then assess recovery of drug-induced expression changes using log fold-change (logFC) correlation and top-200 gene overlap \citep{liang2025scppdm}. The table also reports generation time, and metric definitions
are in Appendix~\ref{app:scrna-details}.

\begin{table}[t]
\centering
\small
\setlength{\tabcolsep}{3pt}
\resizebox{\linewidth}{!}{%
\begin{tabular}{lcccccc}
\toprule
Method & Coupling & NFE & SW$_2\downarrow$ & logFC $r\uparrow$
& Top overlap$\uparrow$ & Time (s)$\downarrow$ \\
\midrule
Count Flow Map & \multirow{4}{*}{Indep.} & 1
& $0.418\pm0.006$ & $\mathbf{0.790\pm0.002}$
& $\mathbf{0.406\pm0.001}$ & $0.21\pm0.00$ \\
Count Flow Map &  & 16
& $0.422\pm0.004$ & $0.775\pm0.007$
& $0.394\pm0.002$ & $2.20\pm0.01$ \\
Count-FM (unit jump) &  & 128--512
& $0.424\pm0.006$ & $0.783\pm0.004$
& $0.395\pm0.003$ & $11.97\pm10.30$ \\
Count-FM ($\tau$-leap) &  & 128
& $0.430\pm0.007$ & $0.786\pm0.006$
& $0.394\pm0.005$ & $6.89\pm0.01$ \\
\midrule
Count Flow Map & \multirow{4}{*}{OT} & 1
& $\mathbf{0.401\pm0.007}$ & $0.778\pm0.007$
& $0.391\pm0.007$ & $0.21\pm0.00$ \\
Count Flow Map &  & 16
& $0.425\pm0.014$ & $0.763\pm0.017$
& $0.390\pm0.004$ & $2.18\pm0.01$ \\
Count-FM (unit jump) &  & 64--256
& $0.432\pm0.014$ & $0.761\pm0.005$
& $0.381\pm0.002$ & $7.01\pm4.54$ \\
Count-FM ($\tau$-leap) &  & 128--256
& $0.434\pm0.015$ & $0.762\pm0.004$
& $0.376\pm0.007$ & $9.13\pm3.91$ \\
\midrule
Conditional NB-VAE & -- & 1-pass
& $0.451\pm0.018$ & $0.773\pm0.012$
& $0.393\pm0.014$ & -- \\
CPA & -- & 1-pass
& $2.594\pm0.120$ & $0.342\pm0.020$
& $0.191\pm0.006$ & -- \\
scGen & -- & 1-pass
& $1.787\pm0.016$ & $0.317\pm0.008$
& $0.159\pm0.001$ & -- \\
scVIDR & -- & 1-pass
& $1.950\pm0.026$ & $0.274\pm0.003$
& $0.159\pm0.002$ & -- \\
Sinkhorn OT & -- & 1-pass
& $1.540$ & $0.336$
& $0.196$ & -- \\
Linear dose-response & -- & 1-pass
& $1.113$ & $0.264$
& $0.160$ & -- \\
\bottomrule
\end{tabular}%
}
\caption{Held-out perturbation prediction. Quality metrics are
averaged equally across 24 test conditions. For neural models, we report mean $\pm$ sample standard deviation across three training seeds. These repeated runs assess neural-network training variability, while the non-neural Linear and Sinkhorn baselines are evaluated once. Count-FM's NFE is selected on validation for each seed. Ranges show the lowest and highest selected NFE.
Time is the number of seconds needed to generate 9,600 cells,
excluding setup and metric computation. A dash indicates that comparable timing is unavailable. Evaluation and timing details can be found in Appendix~\ref{app:scrna-details}.}
\label{tab:scrna-endpoint-app}
\end{table}

With OT coupling, one-step Count Flow Map achieves the lowest
mean sliced $W_2$ ($0.401\pm0.007$), generating 9,600 cells in
approximately $0.21$ seconds. Under independence coupling,
one-step Count Flow Map achieves the highest mean logFC correlation
($0.790\pm0.002$) and top-gene overlap ($0.406\pm0.001$) among
the reported methods. OT improves one-step distributional accuracy,
whereas independence coupling better recovers gene responses.
Increasing from 1 to 16 NFE does not improve any of these three
mean quality metrics under either coupling.

% With OT coupling, one-step Count Flow Map achieves the lowest
% mean sliced $W_2$ ($0.394\pm0.002$), generating 9,600 cells in
% approximately $0.20$ seconds. Its logFC correlation ($0.779$)
% and top-gene overlap ($0.389$) are close to the best reported
% means ($0.786$ and $0.396$). OT improves its one-step
% distributional accuracy and top-gene overlap while leaving mean
% logFC correlation approximately unchanged. Increasing from
% 1 to 16 NFE improves sliced $W_2$ under independence coupling
% but worsens it under OT. Approximate compositional consistency may allow transition errors
% to accumulate differently
% under these couplings. Overall, Count Flow Map combines strong
% distributional accuracy and competitive response recovery with
% very few model evaluations.

\paragraph{Inference efficiency and effect recovery.}
Count Flow Map achieves high distributional accuracy in a single sampling step and generates cells substantially faster than the validation-selected Count-FM baselines (Figure~\ref{fig:scrna-efficiency}A and B).
Panel C further shows that it captures the relative strength of drug responses across the 24 test conditions, measured against matched DMSO controls.
For the fixed run shown here, observed and generated response strengths are strongly correlated under both couplings ($r=0.81$).
The fitted slopes are below the ideal value of $1$, at $0.79$ under independence coupling and $0.80$ under OT, indicating that the model tends to underestimate response magnitudes.
Together, these demonstrate efficient conditional generation that captures variation in drug-response strength across held-out conditions.
Appendix~\ref{app:scrna-examples} examines gene-response patterns by condition at both sampling budgets.

\begin{figure}[t]
    \centering
    \includegraphics[width=\linewidth]{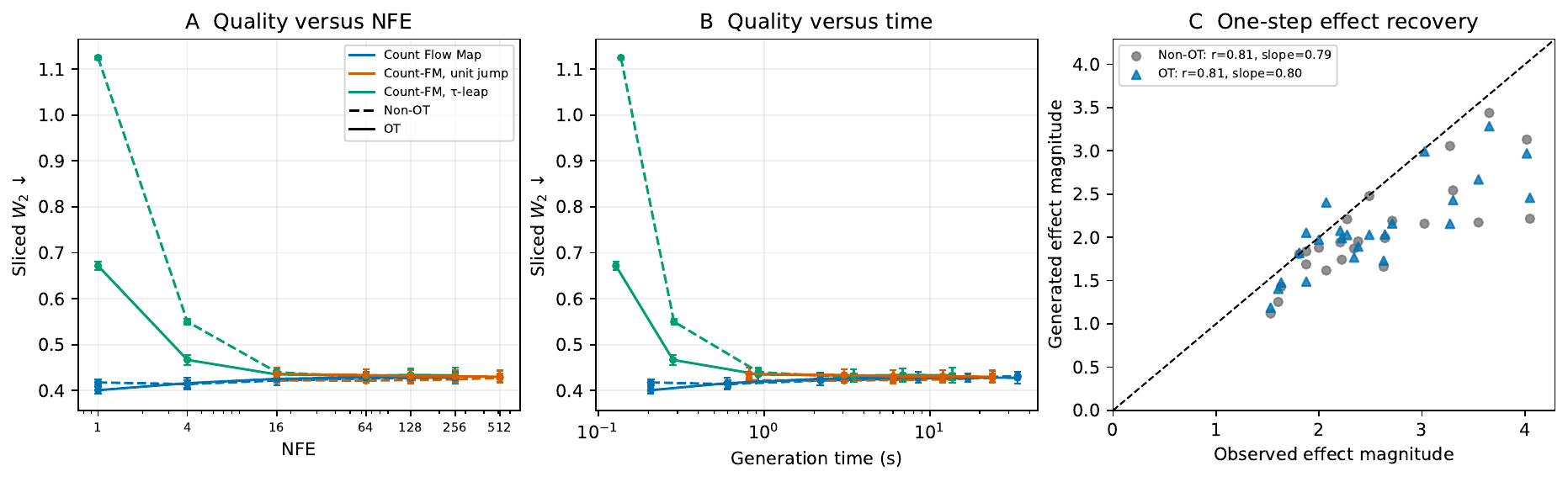}
    \caption{Held-out single-cell prediction and generation cost. A. Sliced $W_2$ versus NFE. B. Sliced $W_2$ versus generation time for 9,600 cells. C. Observed versus generated perturbation strengths at
    1 NFE under both couplings, using one fixed training seed.
    Each point represents one of the 24 test conditions.
    The dashed line indicates perfect agreement.
    $r$ is Pearson correlation across condition magnitudes,
    and the slope is fitted through the origin.
    In A--B, error bars show one sample standard deviation
    across three training seeds, colors identify the model
    or sampler, and solid and dashed lines indicate OT and
    independence coupling, respectively.
    Effect definitions are provided in
    Appendix~\ref{app:scrna-examples}.}
    \label{fig:scrna-efficiency}
\end{figure}

\subsection{Neural Population Forecasting}
\label{sec:application-neural}

We next study temporal prediction of neural population activity using six 50\,ms-binned recordings from the SpikeProphecy release \citep{steinmetz2019distributed,spikeprophecy2026data}. The recordings contain 60,569--68,927 population count vectors and 490--1,104 neurons each, with 385,019 count vectors in total.
The task is to forecast the joint population count vector in the next 50\,ms bin from the preceding 500\,ms of spike-count history. Within each recording, let $Y_b\in\mathbb{N}_0^D$ denote the population vector at bin $b$, where $D$ is the recording-specific number of neurons, and let $H_b=(Y_{b-9},\ldots,Y_b)$ denote the 10-bin history window. We model the next-bin forecast as $\widehat Y_{b+1}\sim K_{\theta,0,\tau}(\cdot\mid Y_b,H_b)$, where $[0,\tau]$ denotes flow time rather than physical time. Because the neuron dimension differs across recordings, we train a separate model for each recording using a chronological 70\%/15\%/15\% train/validation/test split. The models condition only on spike-count history. 
We evaluate 4,096 test histories per recording, with 64 generated draws for each history. 
We compute metrics separately for each recording, average them equally across the six recordings within each training seed, and report the mean and standard deviation across three seeds.

\paragraph{Forecast quality and computational cost.}
Table~\ref{tab:neural-main} compares the count-flow models at
1 and 16 NFE with three direct probabilistic forecasters.
Dimension-normalized energy score assesses the joint count
distribution, while the continuous ranked probability score
(CRPS) evaluates the distribution of mean count per neuron
\citep{gneiting2007proper}.
At 1 NFE, Count Flow Map has lower energy score and population
CRPS than either Count-FM sampler.
At 16 NFE, both Count-FM samplers attain slightly lower
population CRPS than one-step Count Flow Map, but take
$4.2$--$5.1\times$ as long.
Increasing Count Flow Map from 1 to 16 NFE does not improve
either mean quality score.
Among all forecasters, the direct count mixture achieves the
lowest energy score and is faster than one-step Count Flow Map,
while their population CRPS values are similar.
These results show that Count Flow Map reaches forecast quality comparable to multi-step Count-FM in a single step, substantially reducing generation time.
The full NFE comparison, including neuron-wise root mean squared error (RMSE), can be found
in Table~\ref{tab:neural-full} in Appendix~\ref{app:neural}.

\begin{table}[t]
\centering
\small
\setlength{\tabcolsep}{3pt}
\begin{tabular}{lrccc}
\toprule
Method & NFE & Energy score$\downarrow$
& Pop. CRPS$\downarrow$ & Time (ms)$\downarrow$ \\
\midrule
Count Flow Map & 1
& $0.3922\pm0.0003$ & $0.02248\pm0.00020$
& $2.33\pm0.01$ \\
& 16
& $0.3931\pm0.0001$ & $0.02250\pm0.00013$
& $26.06\pm0.10$ \\
\midrule
Count-FM (unit jump) & 1
& $0.4229\pm0.0001$ & $0.03508\pm0.00012$
& $1.26\pm0.01$ \\
& 16
& $0.3948\pm0.0001$ & $0.02220\pm0.00016$
& $9.70\pm0.01$ \\
\midrule
Count-FM ($\tau$-leap) & 1
& $0.3955\pm0.0003$ & $0.03759\pm0.00029$
& $1.41\pm0.00$ \\
& 16
& $0.3934\pm0.0001$ & $0.02222\pm0.00017$
& $11.96\pm0.02$ \\
\midrule
Direct count mixture & --
& $0.3691\pm0.0002$ & $0.02258\pm0.00009$
& $1.68\pm0.01$ \\
Transformer-NB & --
& $0.3709\pm0.0003$ & $0.02469\pm0.00047$
& $1.78\pm0.01$ \\
Population Poisson GLM & --
& $0.3935\pm0.0034$ & $0.02915\pm0.00026$
& $0.29\pm0.00$ \\
\bottomrule
\end{tabular}
\caption{Next-bin neural population forecasting.
Entries are mean $\pm$ sample standard deviation across three seeds, after averaging equally across six recordings.
Time is the average time per input history to generate 64 next-bin population count samples.
NFE is shown for count-flow samplers. The other methods sample directly.}
\label{tab:neural-main}
\end{table}

\paragraph{Dependence in generated neural counts.}
We further examine dependence across neurons in Count Flow Map's generated counts by comparing its original forecasts with a shuffled control. For each fixed input history, the control independently shuffles each neuron's counts across generated draws. This preserves each neuron's predicted distribution and the mean population count, but can change the probability of high-activity events. For each recording, these events are defined by total spike counts above the 95th percentile of training population totals. The fraction of generated totals exceeding this threshold gives the predicted event probability, which is evaluated against the actual held-out outcome using the Brier score \citep{brier1950verification}. The original forecasts have lower mean scores than the shuffled control by $0.00101$ ($3.1\%$) at 1 NFE and $0.00117$ ($3.6\%$) at 16 NFE. Retaining dependence among predicted counts therefore improves event prediction on average, although the benefit varies across recordings. Table~\ref{tab:neural-events} in Appendix~\ref{app:neural-diagnostics} reports the paired comparison.

\section{Conclusion}
\label{sec:conclusion}

In this work, we introduced Count Flow Map, a finite-time extension of count-FM for efficient generation and transport directly in count space. By learning transitions that replace many local birth--death updates, Count Flow Map supports one- or few-step generation while keeping samples nonnegative and integer-valued without a predefined maximum count. We developed a Poisson--Binomial transition model trained through diagonal rate matching and off-diagonal self-distillation.
Across simulations, single-cell drug perturbation prediction, and neural population forecasting, Count Flow Map achieved competitive distributional quality with one or a few model evaluations.
% The single-cell results further demonstrated generation for held-out dose conditions with strong distributional accuracy and competitive perturbation-response recovery.
% The neural results supported efficient one-step forecasting
% within the count-flow family, with a modest benefit from
% retaining generated cross-neuron dependence when predicting
% high-activity events.

Although Count Flow Map performed well across our experiments, the current formulation still has several limitations. First, the transition family and conditional bridge could be made more flexible. Richer count-valued kernels may better capture complex transitions over long intervals while preserving the same local dynamics. Second, the signed-binomial bridge has no conditional variability in coordinates whose endpoint counts coincide. Adding independent Poisson fluctuations with mean $t(1-t)$, for example, could introduce intermediate variability while preserving the endpoints, but would require deriving new conditional target rates. Third, compositional consistency is enforced approximately through self-distillation, so improving this consistency could reduce discrepancies between direct and multi-step generation. Self-conditioning \citep{chen2023analog-bits} and fixed-point flow formulations \citep{yoo2026fixed-point-flows} suggest possible extensions, although their adaptation to count-valued kernels remains to be investigated. Finally, the signed-binomial bridge inherits the endpoint singularity of count-FM \citep{wei2026count-flow-matching}, so the current sampler truncates before $t=1$. Count bridges with better endpoint behavior \citep{fishman2026count-bridges} could reduce the resulting truncation error.

In summary, Count Flow Map provides a direct and efficient approach to count-valued generation and transport. By extending local birth and death dynamics to finite-time transitions, it enables one- or few-step sampling while preserving nonnegative integer counts. This framework also provides a foundation for future extensions, such as constrained transport inspired by Riemannian flow matching \citep{chen2024riemannian-flow-matching} and inference-time steering through particle reweighting and resampling \citep{singhal2025fk-steering}.

\begin{ack}
GW was supported by the Center for Computational Neuroscience (CCN) at the Flatiron Institute, Simons Foundation.
\end{ack}

\bibliographystyle{abbrvnat}
\bibliography{count_flow_map_method}

@article{boffi2025flow-map-matching,
  title   = {Flow Map Matching with Stochastic Interpolants: A Mathematical Framework for Consistency Models},
  author  = {Boffi, Nicholas Matthew and Albergo, Michael Samuel and Vanden-Eijnden, Eric},
  journal = {Transactions on Machine Learning Research},
  year    = {2025},
  url     = {https://openreview.net/forum?id=cqDH0e6ak2},
}

@inproceedings{boffi2025flow-map-self-distillation,
  title     = {How to Build a Consistency Model: Learning Flow Maps via Self-Distillation},
  author    = {Boffi, Nicholas M. and Albergo, Michael S. and Vanden-Eijnden, Eric},
  booktitle = {Advances in Neural Information Processing Systems},
  volume    = {38},
  pages     = {37559--37595},
  year      = {2025},
  doi       = {10.52202/085713-1121},
  url       = {https://papers.nips.cc/paper_files/paper/2025/hash/2ff3605c9e442d55156d99fe4f5f228d-Abstract-Conference.html},
}

@misc{potaptchik2026discrete-flow-maps,
  title        = {Discrete Flow Maps},
  author       = {Potaptchik, Peter and Yim, Jason and Saravanan, Adhi and Holderrieth, Peter and Vanden-Eijnden, Eric and Albergo, Michael S.},
  howpublished = {arXiv preprint arXiv:2604.09784},
  year         = {2026},
  eprint       = {2604.09784},
  archivePrefix= {arXiv},
  primaryClass = {stat.ML},
  doi          = {10.48550/arXiv.2604.09784},
  url          = {https://arxiv.org/abs/2604.09784}
}

@inproceedings{roos2026categorical-flow-maps,
  title     = {Categorical Flow Maps},
  author    = {Roos, Daan and Davis, Oscar and Eijkelboom, Floor and Bronstein, Michael and Welling, Max and Ceylan, {\.{I}}smail {\.{I}}lkan and Ambrogioni, Luca and {van de Meent}, Jan-Willem},
  booktitle = {International Conference on Machine Learning},
  year      = {2026},
  url       = {https://arxiv.org/abs/2602.12233}
}

@inproceedings{song2023consistency-models,
  title     = {Consistency Models},
  author    = {Song, Yang and Dhariwal, Prafulla and Chen, Mark and Sutskever, Ilya},
  booktitle = {Proceedings of the 40th International Conference on Machine Learning},
  editor    = {Krause, Andreas and Brunskill, Emma and Cho, Kyunghyun and Engelhardt, Barbara and Sabato, Sivan and Scarlett, Jonathan},
  series    = {Proceedings of Machine Learning Research},
  volume    = {202},
  pages     = {32211--32252},
  publisher = {PMLR},
  year      = {2023},
  url       = {https://proceedings.mlr.press/v202/song23a.html},
}

@inproceedings{salimans2022progressive-distillation,
  title     = {Progressive Distillation for Fast Sampling of Diffusion Models},
  author    = {Salimans, Tim and Ho, Jonathan},
  booktitle = {International Conference on Learning Representations},
  year      = {2022},
  url       = {https://openreview.net/forum?id=TIdIXIpzhoI},
}

@inproceedings{kim2024consistency-trajectory,
  title     = {Consistency Trajectory Models: Learning Probability Flow {ODE} Trajectory of Diffusion},
  author    = {Kim, Dongjun and Lai, Chieh-Hsin and Liao, Wei-Hsiang and Murata, Naoki and Takida, Yuhta and Uesaka, Toshimitsu and He, Yutong and Mitsufuji, Yuki and Ermon, Stefano},
  booktitle = {International Conference on Learning Representations},
  year      = {2024},
  url       = {https://proceedings.iclr.cc/paper_files/paper/2024/hash/c204d12afa0175285e5aac65188808b4-Abstract-Conference.html},
}

@inproceedings{frans2025shortcut-models,
  title     = {One Step Diffusion via Shortcut Models},
  author    = {Frans, Kevin and Hafner, Danijar and Levine, Sergey and Abbeel, Pieter},
  booktitle = {International Conference on Learning Representations},
  year      = {2025},
  url       = {https://openreview.net/forum?id=OlzB6LnXcS},
}

@inproceedings{lipman2023flow-matching,
  title     = {Flow Matching for Generative Modeling},
  author    = {Lipman, Yaron and Chen, Ricky T. Q. and Ben-Hamu, Heli and Nickel, Maximilian and Le, Matt},
  booktitle = {International Conference on Learning Representations},
  year      = {2023},
  url       = {https://openreview.net/forum?id=PqvMRDCJT9t},
}

@article{albergo2025stochastic-interpolants,
  title   = {Stochastic Interpolants: A Unifying Framework for Flows and Diffusions},
  author  = {Albergo, Michael S. and Boffi, Nicholas M. and Vanden-Eijnden, Eric},
  journal = {Journal of Machine Learning Research},
  volume  = {26},
  number  = {209},
  pages   = {1--80},
  year    = {2025},
  url     = {https://jmlr.org/papers/v26/23-1605.html},
}

@inproceedings{holderrieth2025generator-matching,
  title     = {Generator Matching: Generative Modeling with Arbitrary {Markov} Processes},
  author    = {Holderrieth, Peter and Havasi, Marton and Yim, Jason and Shaul, Neta and Gat, Itai and Jaakkola, Tommi and Karrer, Brian and Chen, Ricky T. Q. and Lipman, Yaron},
  booktitle = {International Conference on Learning Representations},
  year      = {2025},
  url       = {https://proceedings.iclr.cc/paper_files/paper/2025/hash/819aaee144cb40e887a4aa9e781b1547-Abstract-Conference.html},
}

@inproceedings{wei2026count-flow-matching,
  title     = {Flow Matching for Count Data},
  author    = {Wei, Ganchao and Pearson, John},
  booktitle = {The Fortieth Annual Conference on Neural Information Processing Systems},
  year      = {2026},
  url       = {https://openreview.net/forum?id=5N0iGXXzEJ}
}

@inproceedings{campbell2022continuous-time-discrete,
  title     = {A Continuous Time Framework for Discrete Denoising Models},
  author    = {Campbell, Andrew and Benton, Joe and De Bortoli, Valentin and Rainforth, Thomas and Deligiannidis, George and Doucet, Arnaud},
  booktitle = {Advances in Neural Information Processing Systems},
  volume    = {35},
  pages     = {28266--28279},
  year      = {2022},
  doi       = {10.52202/068431-2049},
  url       = {https://proceedings.neurips.cc/paper_files/paper/2022/hash/b5b528767aa35f5b1a60fe0aaeca0563-Abstract-Conference.html},
}

@inproceedings{gat2024discrete-flow-matching,
  title     = {Discrete Flow Matching},
  author    = {Gat, Itai and Remez, Tal and Shaul, Neta and Kreuk, Felix and Chen, Ricky T. Q. and Synnaeve, Gabriel and Adi, Yossi and Lipman, Yaron},
  booktitle = {Advances in Neural Information Processing Systems},
  volume    = {37},
  pages     = {133345--133385},
  year      = {2024},
  doi       = {10.52202/079017-4239},
  url       = {https://proceedings.neurips.cc/paper_files/paper/2024/hash/f0d629a734b56a642701bba7bc8bb3ed-Abstract-Conference.html},
}

@incollection{arampatzis2015pathwise-sensitivity,
  title     = {Pathwise Sensitivity Analysis in Transient Regimes},
  author    = {Arampatzis, Georgios and Katsoulakis, Markos A. and Pantazis, Yannis},
  booktitle = {Stochastic Equations for Complex Systems: Theoretical and Computational Topics},
  editor    = {Heinz, Stefan and Bessaih, Hakima},
  series    = {Mathematical Engineering},
  pages     = {105--124},
  publisher = {Springer International Publishing},
  address   = {Cham},
  year      = {2015},
  isbn      = {978-3-319-18206-3},
  doi       = {10.1007/978-3-319-18206-3_5},
  url       = {https://doi.org/10.1007/978-3-319-18206-3_5},
}

@book{norris1997markov-chains,
  title     = {Markov Chains},
  author    = {Norris, J. R.},
  series    = {Cambridge Series in Statistical and Probabilistic Mathematics},
  publisher = {Cambridge University Press},
  address   = {Cambridge},
  year      = {1997},
  isbn      = {978-0-521-48181-6},
  doi       = {10.1017/CBO9780511810633},
  url       = {https://doi.org/10.1017/CBO9780511810633}
}

@inproceedings{ho2020denoising-diffusion,
  title     = {Denoising Diffusion Probabilistic Models},
  author    = {Ho, Jonathan and Jain, Ajay and Abbeel, Pieter},
  booktitle = {Advances in Neural Information Processing Systems},
  volume    = {33},
  pages     = {6840--6851},
  year      = {2020},
  url       = {https://proceedings.neurips.cc/paper/2020/hash/4c5bcfec8584af0d967f1ab10179ca4b-Abstract.html}
}

@inproceedings{austin2021d3pm,
  title     = {Structured Denoising Diffusion Models in Discrete State-Spaces},
  author    = {Austin, Jacob and Johnson, Daniel D. and Ho, Jonathan and Tarlow, Daniel and van den Berg, Rianne},
  booktitle = {Advances in Neural Information Processing Systems},
  volume    = {34},
  pages     = {17981--17993},
  year      = {2021},
  url       = {https://proceedings.neurips.cc/paper/2021/hash/958c530554f78bcd8e97125b70e6973d-Abstract.html}
}

@inproceedings{fishman2026count-bridges,
  title     = {{Count Bridges} enable Modeling and Deconvolving Transcriptomic Data},
  author    = {Fishman, Nic and Gowri, Gokul and Kumar, Tanush and Lu, Jiaqi and de Bortoli, Valentin and Gootenberg, Jonathan S. and Abudayyeh, Omar},
  booktitle = {International Conference on Learning Representations},
  year      = {2026},
  url       = {https://arxiv.org/abs/2603.04730}
}

@inproceedings{soatto2026countsdiff,
  title     = {{CountsDiff}: A Diffusion Model on the Natural Numbers for Generation and Imputation of Count-Based Data},
  author    = {Soatto, Renzo G. and Hoel, Anders and Ren, Greycen and Alam, Shorna and Bates, Stephen and Daskalakis, Nikolaos P. and Uhler, Caroline and Skoularidou, Maria},
  booktitle = {Proceedings of the 43rd International Conference on Machine Learning},
  year      = {2026},
  url       = {https://arxiv.org/abs/2604.03779}
}

@article{tian2004binomial-leap,
  title   = {Binomial leap methods for simulating stochastic chemical kinetics},
  author  = {Tian, Tianhai and Burrage, Kevin},
  journal = {The Journal of Chemical Physics},
  volume  = {121},
  number  = {21},
  pages   = {10356--10364},
  year    = {2004},
  doi     = {10.1063/1.1810475},
  url     = {https://doi.org/10.1063/1.1810475}
}

@article{zhang2025tahoe100m,
  title   = {{Tahoe-100M}: A Giga-Scale Single-Cell Perturbation Atlas for Context-Dependent Gene Function and Cellular Modeling},
  author  = {Zhang, Jesse and Ubas, Airol A. and de Borja, Richard and Svensson, Valentine and Thomas, Nicole and Thakar, Neha and others},
  journal = {bioRxiv},
  year    = {2025},
  doi     = {10.1101/2025.02.20.639398},
  url     = {https://www.biorxiv.org/content/10.1101/2025.02.20.639398v3},
  note    = {Preprint}
}

@inproceedings{sohn2015conditional-vae,
  title     = {Learning Structured Output Representation using Deep Conditional Generative Models},
  author    = {Sohn, Kihyuk and Lee, Honglak and Yan, Xinchen},
  booktitle = {Advances in Neural Information Processing Systems},
  volume    = {28},
  year      = {2015},
  url       = {https://proceedings.neurips.cc/paper_files/paper/2015/file/8d55a249e6baa5c06772297520da2051-Paper.pdf}
}

@article{lopez2018scvi,
  title   = {Deep generative modeling for single-cell transcriptomics},
  author  = {Lopez, Romain and Regier, Jeffrey and Cole, Michael B. and Jordan, Michael I. and Yosef, Nir},
  journal = {Nature Methods},
  volume  = {15},
  number  = {12},
  pages   = {1053--1058},
  year    = {2018},
  doi     = {10.1038/s41592-018-0229-2},
  url     = {https://www.nature.com/articles/s41592-018-0229-2}
}

@article{lotfollahi2023cpa,
  title   = {Predicting cellular responses to complex perturbations in high-throughput screens},
  author  = {Lotfollahi, Mohammad and Klimovskaia Susmelj, Anna and De Donno, Carlo and Hetzel, Leon and Ji, Yuge and Ibarra, Ignacio L. and Srivatsan, Sanjay R. and Naghipourfar, Mohsen and Daza, Riza M. and Martin, Beth and Shendure, Jay and McFaline-Figueroa, Jose L. and Boyeau, Pierre and Wolf, F. Alexander and Yakubova, Nafissa and G{\"u}nnemann, Stephan and Trapnell, Cole and Lopez-Paz, David and Theis, Fabian J.},
  journal = {Molecular Systems Biology},
  volume  = {19},
  number  = {6},
  pages   = {e11517},
  year    = {2023},
  doi     = {10.15252/msb.202211517},
  url     = {https://link.springer.com/article/10.15252/msb.202211517}
}

@article{lotfollahi2019scgen,
  title   = {{scGen} predicts single-cell perturbation responses},
  author  = {Lotfollahi, Mohammad and Wolf, F. Alexander and Theis, Fabian J.},
  journal = {Nature Methods},
  volume  = {16},
  number  = {8},
  pages   = {715--721},
  year    = {2019},
  doi     = {10.1038/s41592-019-0494-8},
  url     = {https://www.nature.com/articles/s41592-019-0494-8}
}

@article{kana2023scvidr,
  title   = {Generative modeling of single-cell gene expression for dose-dependent chemical perturbations},
  author  = {Kana, Omar and Nault, Rance and Filipovic, David and Marri, Daniel and Zacharewski, Tim and Bhattacharya, Sudin},
  journal = {Patterns},
  volume  = {4},
  number  = {8},
  pages   = {100817},
  year    = {2023},
  doi     = {10.1016/j.patter.2023.100817},
  url     = {https://doi.org/10.1016/j.patter.2023.100817}
}

@inproceedings{cuturi2013sinkhorn,
  title     = {{Sinkhorn} Distances: Lightspeed Computation of Optimal Transport},
  author    = {Cuturi, Marco},
  booktitle = {Advances in Neural Information Processing Systems},
  volume    = {26},
  year      = {2013},
  url       = {https://papers.nips.cc/paper/4927-sinkhorn-distances-lightspeed-computation-of-optimal-transport}
}

@misc{yoo2026fixed-point-flows,
  title         = {Self-conditioned Flow Map Language Models via Fixed-point Flows},
  author        = {Yoo, Jaehoon and Kim, Wonjung and Eijkelboom, Floor and Lee, Chanhyuk and Boffi, Nicholas M. and Hong, Seunghoon and Kim, Jinwoo},
  year          = {2026},
  howpublished  = {arXiv preprint arXiv:2607.00714},
  eprint        = {2607.00714},
  archivePrefix = {arXiv},
  primaryClass  = {cs.CL},
  doi           = {10.48550/arXiv.2607.00714},
  url           = {https://arxiv.org/abs/2607.00714}
}

@article{feinberg2014kolmogorov,
  title   = {On solutions of {Kolmogorov}'s equations for nonhomogeneous jump {Markov} processes},
  author  = {Feinberg, Eugene A. and Mandava, Manasa and Shiryaev, Albert N.},
  journal = {Journal of Mathematical Analysis and Applications},
  volume  = {411},
  number  = {1},
  pages   = {261--270},
  year    = {2014},
  doi     = {10.1016/j.jmaa.2013.09.043},
  url     = {https://doi.org/10.1016/j.jmaa.2013.09.043}
}

@article{gretton2012kernel-test,
  title   = {A Kernel Two-Sample Test},
  author  = {Gretton, Arthur and Borgwardt, Karsten M. and Rasch, Malte J. and Sch{\"o}lkopf, Bernhard and Smola, Alexander},
  journal = {Journal of Machine Learning Research},
  volume  = {13},
  number  = {25},
  pages   = {723--773},
  year    = {2012},
  url     = {https://jmlr.org/papers/v13/gretton12a.html}
}

@article{bonneel2015sliced-wasserstein,
  title   = {Sliced and {Radon} {Wasserstein} Barycenters of Measures},
  author  = {Bonneel, Nicolas and Rabin, Julien and Peyr{\'e}, Gabriel and Pfister, Hanspeter},
  journal = {Journal of Mathematical Imaging and Vision},
  volume  = {51},
  number  = {1},
  pages   = {22--45},
  year    = {2015},
  doi     = {10.1007/s10851-014-0506-3},
  url     = {https://doi.org/10.1007/s10851-014-0506-3}
}

@inproceedings{chen2023analog-bits,
  title     = {{Analog Bits}: Generating Discrete Data using Diffusion Models with Self-Conditioning},
  author    = {Chen, Ting and Zhang, Ruixiang and Hinton, Geoffrey},
  booktitle = {The Eleventh International Conference on Learning Representations},
  year      = {2023},
  url       = {https://openreview.net/forum?id=3itjR9QxFw}
}

@article{alosh1987inar,
  title   = {First-order integer-valued autoregressive ({INAR}(1)) process},
  author  = {Al-Osh, M. A. and Alzaid, A. A.},
  journal = {Journal of Time Series Analysis},
  volume  = {8},
  number  = {3},
  pages   = {261--275},
  year    = {1987},
  doi     = {10.1111/j.1467-9892.1987.tb00438.x}
}

@article{tong2024ot-cfm,
  title = {Improving and Generalizing Flow-Based Generative Models
           with Minibatch Optimal Transport},
  author = {Tong, Alexander and Fatras, Kilian and Malkin, Nikolay
            and Huguet, Guillaume and Zhang, Yanlei
            and Rector-Brooks, Jarrid and Wolf, Guy
            and Bengio, Yoshua},
  journal = {Transactions on Machine Learning Research},
  year = {2024},
  url = {https://arxiv.org/abs/2302.00482}
}

@misc{liang2025scppdm,
  title = {{scPPDM}: A Diffusion Model for Single-Cell Drug-Response Prediction},
  author = {Liang, Zhaokang and Zhuang, Shuyang and Jiao, Xiaoran
            and Mao, Weian and Chen, Hao and Shen, Chunhua},
  year = {2025},
  eprint = {2510.11726},
  archivePrefix = {arXiv},
  primaryClass = {q-bio.QM},
  url = {https://arxiv.org/abs/2510.11726}
}

@article{steinmetz2019distributed,
  author = {Steinmetz, Nicholas A. and Zatka-Haas, Peter
            and Carandini, Matteo and Harris, Kenneth D.},
  title = {Distributed coding of choice, action and engagement
           across the mouse brain},
  journal = {Nature},
  volume = {576},
  number = {7786},
  pages = {266--273},
  year = {2019},
  doi = {10.1038/s41586-019-1787-x},
  url = {https://doi.org/10.1038/s41586-019-1787-x}
}

@misc{spikeprophecy2026data,
  title         = {{SpikeProphecy}: A Large-Scale Benchmark for Autoregressive Neural Population Forecasting},
  author        = {Minnick, John R. and Geng, Jinghui and Hussain, Kamran
                   and Gonzalez-Ferrer, Jesus and Robbins, Ash
                   and Mostajo-Radji, Mohammed A. and Haussler, David
                   and Eshraghian, Jason K. and Teodorescu, Mircea},
  year          = {2026},
  howpublished  = {arXiv preprint arXiv:2605.12992},
  eprint        = {2605.12992},
  archivePrefix = {arXiv},
  primaryClass  = {q-bio.NC},
  url           = {https://arxiv.org/abs/2605.12992},
  note          = {Processed data: \url{https://huggingface.co/datasets/mysteriousauthor/spikeprophecy-steinmetz}}
}

@article{gneiting2007proper,
  author = {Gneiting, Tilmann and Raftery, Adrian E.},
  title = {Strictly proper scoring rules, prediction,
           and estimation},
  journal = {Journal of the American Statistical Association},
  volume = {102},
  number = {477},
  pages = {359--378},
  year = {2007},
  doi = {10.1198/016214506000001437},
  url = {https://doi.org/10.1198/016214506000001437}
}

@article{ferro2014fair,
  author = {Ferro, Christopher A. T.},
  title = {Fair scores for ensemble forecasts},
  journal = {Quarterly Journal of the Royal Meteorological Society},
  volume = {140},
  number = {683},
  pages = {1917--1923},
  year = {2014},
  doi = {10.1002/qj.2270},
  url = {https://doi.org/10.1002/qj.2270}
}

@article{brier1950verification,
  author = {Brier, Glenn W.},
  title = {Verification of forecasts expressed in terms
           of probability},
  journal = {Monthly Weather Review},
  volume = {78},
  number = {1},
  pages = {1--3},
  year = {1950},
  doi = {10.1175/1520-0493(1950)078<0001:VOFEIT>2.0.CO;2},
  url = {https://journals.ametsoc.org/view/journals/mwre/78/1/1520-0493_1950_078_0001_vofeit_2_0_co_2.xml}
}

@inproceedings{chen2024riemannian-flow-matching,
  title     = {Flow Matching on General Geometries},
  author    = {Chen, Ricky T. Q. and Lipman, Yaron},
  booktitle = {International Conference on Learning Representations},
  year      = {2024},
  url       = {https://proceedings.iclr.cc/paper_files/paper/2024/hash/d1f9936d3be6997ffffab692977eebe6-Abstract-Conference.html}
}

@inproceedings{singhal2025fk-steering,
  title     = {A General Framework for Inference-time Scaling and Steering of Diffusion Models},
  author    = {Singhal, Raghav and Horvitz, Zachary and Teehan, Ryan and Ren, Mengye and Yu, Zhou and McKeown, Kathleen and Ranganath, Rajesh},
  booktitle = {Proceedings of the 42nd International Conference on Machine Learning},
  series    = {Proceedings of Machine Learning Research},
  volume    = {267},
  pages     = {55810--55827},
  publisher = {PMLR},
  year      = {2025},
  url       = {https://proceedings.mlr.press/v267/singhal25b.html}
}

% \section*{References}

% References follow the acknowledgments in the camera-ready paper. Use unnumbered first-level heading for
% the references. Any choice of citation style is acceptable as long as you are
% consistent. It is permissible to reduce the font size to \verb+small+ (9 point)
% when listing the references.
% Note that the Reference section does not count towards the page limit.
% \medskip

% {
% \small

% [1] Alexander, J.A.\ \& Mozer, M.C.\ (1995) Template-based algorithms for
% connectionist rule extraction. In G.\ Tesauro, D.S.\ Touretzky and T.K.\ Leen
% (eds.), {\it Advances in Neural Information Processing Systems 7},
% pp.\ 609--616. Cambridge, MA: MIT Press.

% [2] Bower, J.M.\ \& Beeman, D.\ (1995) {\it The Book of GENESIS: Exploring
%   Realistic Neural Models with the GEneral NEural SImulation System.}  New York:
% TELOS/Springer--Verlag.

% [3] Hasselmo, M.E., Schnell, E.\ \& Barkai, E.\ (1995) Dynamics of learning and
% recall at excitatory recurrent synapses and cholinergic modulation in rat
% hippocampal region CA3. {\it Journal of Neuroscience} {\bf 15}(7):5249-5262.
% }

%%%%%%%%%%%%%%%%%%%%%%%%%%%%%%%%%%%%%%%%%%%%%%%%%%%%%%%%%%%%

\appendix

\section{Additional Method Details and Theory}
\label{app:method-theory}

\subsection{Consistency rules for Count Flow Map}
\label{app:flow-map-consistency-proof}

We justify the three consistency rules in
Equation~\eqref{eq:count-flow-map-consistency} and
Proposition~\ref{prop:count-flow-map-characterization}.
Let $Q_t$ be a time-dependent Markov jump generator on
$\mathcal{X}=\mathbb{N}_0^d$. Assume that the associated process is
nonexplosive, admits a unique transition family $P^Q_{s,t}$, and satisfies
the regularity needed for the derivatives below.
See \citet{norris1997markov-chains} for continuous-time Markov-chain
background and \citet{feinberg2014kolmogorov} for Kolmogorov equations and
uniqueness results for nonhomogeneous jump Markov processes.

\paragraph{Compositional consistency.}
For $s\le u\le t$, the Markov property gives
\[
P^Q_{s,t}(y\mid x)
=
\sum_{z\in\mathcal{X}}
P^Q_{s,u}(z\mid x)P^Q_{u,t}(y\mid z),
\]
which is exactly $P^Q_{s,t}=P^Q_{s,u}P^Q_{u,t}$.

\paragraph{Lagrangian (forward) consistency.}
The diagonal generator satisfies $P^Q_{t,t+h}=I+hQ_t+o(h)$. Combining this expansion with compositional consistency,
\[
P^Q_{s,t+h}=P^Q_{s,t}P^Q_{t,t+h},
\]
gives, for every test function $f$,
\[
\frac{P^Q_{s,t+h}f-P^Q_{s,t}f}{h}
=
P^Q_{s,t}\frac{P^Q_{t,t+h}f-f}{h}
\longrightarrow P^Q_{s,t}Q_tf.
\]
Hence $\partial_tP^Q_{s,t}=P^Q_{s,t}Q_t$.

\paragraph{Eulerian (backward) consistency.}
Similarly,
\[
P^Q_{s,t}=P^Q_{s,s+h}P^Q_{s+h,t}.
\]
Using $P^Q_{s,s+h}=I+hQ_s+o(h)$ and rearranging yields
\[
\frac{P^Q_{s+h,t}-P^Q_{s,t}}{h}
\longrightarrow -Q_sP^Q_{s,t},
\]
so $\partial_sP^Q_{s,t}=-Q_sP^Q_{s,t}$.

These arguments show that the exact transition family satisfies all three consistency rules. Conversely, the forward equation with $K_{s,s}=I$ uniquely determines $P^Q_{s,t}$, and the same is true for the backward equation with $K_{t,t}=I$. Finally, suppose that $K_{s,t}$ satisfies compositional consistency and has diagonal generator $Q_t$. Then
\[
K_{s,t+h}=K_{s,t}K_{t,t+h}
\]
and the same difference-quotient argument used above gives $\partial_tK_{s,t}=K_{s,t}Q_t$. Uniqueness of the forward equation therefore implies $K_{s,t}=P^Q_{s,t}$. This proves Proposition~\ref{prop:count-flow-map-characterization}.

\subsection{Exact likelihood of the count transition kernel}
\label{app:count-kernel-likelihood}

We first specify the finite-time parameterization used in our model.
Let $\Delta=t-s$. Using the local rates
$\lambda_{\theta,i}(x,s)$ and
$\mu_{\theta,i}(x,s)=x_i\beta_{\theta,i}(x,s)$ introduced in
Section~\ref{subsec:count-transition-kernel}, for mixture component $j$ we
parameterize the finite-time birth mean and death probability as
\begin{equation}
\begin{aligned}
a_{j,i}(x,s,t)
&=
\Delta\,\lambda_{\theta,i}(x,s)
\exp\!\left\{
\Delta c^+_{\theta,j,i}(x,s,t)
\right\},\\
q_{j,i}(x,s,t)
&=
1-\exp\!\left[
-\Delta\,\beta_{\theta,i}(x,s)
\exp\!\left\{
\Delta c^-_{\theta,j,i}(x,s,t)
\right\}
\right].
\end{aligned}
\label{eq:finite-time-correction-parameterization}
\end{equation}
The functions $c^+_{\theta,j,i}$ and $c^-_{\theta,j,i}$ modify the
finite-time birth and death behavior. Multiplication by $\Delta$ inside
the exponential makes these corrections vanish to first order near the
diagonal, so that
$a_{j,i}=h\lambda_{\theta,i}+O(h^2)$ and
$q_{j,i}=h\beta_{\theta,i}+O(h^2)$ over an interval of length $h$.

We use a finite mixture to increase the flexibility of the transition
kernel. Let $J$ denote the number of components and let
$w_{\theta,j}(x,s,t)\ge0$, with
$\sum_{j=1}^Jw_{\theta,j}(x,s,t)=1$, denote their weights.
Conditional on component $j$, the coordinates follow the
Poisson--Binomial transition described in
Equation~\eqref{eq:poisson-binomial-transition}, with parameters
$a_{j,i}$ and $q_{j,i}$ defined above.

For fixed $x$, component $j$, and coordinate $i$, the relation $Y_i=x_i-D_i+B_i$ implies that, if $D_i=d$, then $B_i=y_i-x_i+d$. Hence
\begin{equation}
\begin{aligned}
k_{\theta,j,i}(y_i\mid x_i,s,t)
=
\sum_{d=\max(0,x_i-y_i)}^{x_i}
&\binom{x_i}{d}q_{j,i}^{d}(1-q_{j,i})^{x_i-d}\\
&\times
\exp(-a_{j,i})
\frac{a_{j,i}^{\,y_i-x_i+d}}{(y_i-x_i+d)!}.
\end{aligned}
\label{eq:exact-coordinate-count-kernel}
\end{equation}
Conditional on a mixture component, the coordinates are sampled independently. The full kernel is therefore
\begin{equation}
K_{\theta,s,t}(y\mid x)
=
\sum_{j=1}^Jw_{\theta,j}(x,s,t)
\prod_{i=1}^dk_{\theta,j,i}(y_i\mid x_i,s,t).
\label{eq:exact-mixture-count-kernel}
\end{equation}
Each coordinate distribution is normalized, the mixture weights sum to one, and $D_i\le x_i$ almost surely. Thus Equation~\eqref{eq:exact-mixture-count-kernel} defines a valid transition kernel on $\mathbb{N}_0^d$ without a predefined maximum count. Setting $J=1$ recovers the kernel described in the main text.

\subsection{Proof of diagonal generator consistency}
\label{app:diagonal-generator-proof}

% We prove Proposition~\ref{prop:diagonal-generator-consistency}. Fix $x$ and $t$, and consider the interval $[t,t+h]$. For each mixture component $j$, local boundedness of the corrections gives
% \[
% a_{j,i}
% =h\lambda_{\theta,i}(x,t)+O(h^2),
% \qquad
% q_{j,i}
% =h\beta_{\theta,i}(x,t)+O(h^2).
% \]
We prove Proposition~\ref{prop:diagonal-generator-consistency}. Fix $x$
and $t$, and consider the interval $[t,t+h]$. From
Equation~\eqref{eq:finite-time-correction-parameterization}, local
boundedness of $c^+_{\theta,j,i}$ and $c^-_{\theta,j,i}$ gives, for
every mixture component $j$,
\[
a_{j,i}
=
h\lambda_{\theta,i}(x,t)+O(h^2),
\qquad
q_{j,i}
=
h\beta_{\theta,i}(x,t)+O(h^2).
\]
Therefore
\[
\mathbb{P}(B_{j,i}=1)
=h\lambda_{\theta,i}(x,t)+O(h^2),
\qquad
\mathbb{P}(D_{j,i}=1)
=h\mu_{\theta,i}(x,t)+O(h^2),
\]
while the probabilities of two or more births, two or more deaths, simultaneous birth and death, or events in more than one coordinate are all $O(h^2)$. Consequently,
\[
\mathbb{P}(Y=x+e_i\mid x,M=j)
=h\lambda_{\theta,i}(x,t)+O(h^2),
\]
\[
\mathbb{P}(Y=x-e_i\mid x,M=j)
=h\mu_{\theta,i}(x,t)+O(h^2),
\]
and
\[
\mathbb{P}(Y=x\mid x,M=j)
=1-h\sum_i\bigl[\lambda_{\theta,i}(x,t)+\mu_{\theta,i}(x,t)\bigr]+O(h^2).
\]
The first-order terms do not depend on the mixture component, so averaging over the mixture weights leaves them unchanged. At $h=0$, the Poisson means and death probabilities are zero, hence $Y=x$ almost surely. This proves the proposition.

Under local boundedness and sufficient regularity of the jump rates, the
exact transition family generated by $Q_t^\theta$ also satisfies, for each
fixed $(x,t)$,
\[
P^\theta_{t,t+h}(\cdot\mid x)
=
\delta_x+hQ_t^\theta(x,\cdot)+O_{\mathrm{TV}}(h^2).
\]
Since $K_{\theta,t,t+h}$ has the same expansion,
\[
D_{\mathrm{TV}}\!\left(
K_{\theta,t,t+h}(\cdot\mid x),
P^\theta_{t,t+h}(\cdot\mid x)
\right)=O(h^2).
\]

\subsection{Diagonal matching and its relation to count-FM}
\label{app:count-fm-diagonal-matching}

For completeness, we derive the target rates and the population optimum used in the diagonal loss. Given an endpoint pair $(x_0,x_1)$, let $n_i=|x_{1,i}-x_{0,i}|$ and $\sigma_i=\operatorname{sgn}(x_{1,i}-x_{0,i})$. The conditional bridge is $X_t^{(i)}=x_{0,i}+\sigma_iB_t^{(i)}$ with $B_t^{(i)}\sim\operatorname{Binomial}(n_i,t)$.

For a fixed coordinate $i$, let $p_t^{(i)}(x_i\mid x_0,x_1)$ denote its conditional bridge marginal. Local mass preservation gives
\[
\begin{aligned}
\partial_t p_t^{(i)}(x_i\mid x_0,x_1)
={}&
\bar\lambda_{t,i}(x_i-1)p_t^{(i)}(x_i-1\mid x_0,x_1)\\
&+\bar\mu_{t,i}(x_i+1)p_t^{(i)}(x_i+1\mid x_0,x_1)\\
&-\bigl[\bar\lambda_{t,i}(x_i)+\bar\mu_{t,i}(x_i)\bigr]
p_t^{(i)}(x_i\mid x_0,x_1).
\end{aligned}
\]
The first two terms describe birth and death influx into $x_i$, while the last term describes outflux from $x_i$.
Substituting the Binomial bridge gives the rates in Equation~\eqref{eq:conditional-bridge-rates}. Equivalently, the $n_i$ required unit changes can be viewed as occurring independently over $[0,1]$; conditional on not having occurred before time $t$, each remaining event has instantaneous hazard $(1-t)^{-1}$.

Let $p_t(x)$ be the marginal bridge after averaging over $(X_0,X_1)\sim\pi$. The marginal rates
\[
\lambda_i^\star(x,t)
=\mathbb{E}[\bar\lambda_{t,i}\mid X_t=x],
\qquad
\mu_i^\star(x,t)
=\mathbb{E}[\bar\mu_{t,i}\mid X_t=x]
\]
generate $p_t$ by conditional-to-marginal generator matching \citep{holderrieth2025generator-matching}.

It remains to verify that Equation~\eqref{eq:diagonal-rate-loss} identifies these marginal rates.
For any nonnegative target $A$ and any prediction $b(X)>0$ measurable with respect to $X$,
\[
\mathbb{E}[\ell(A,b(X))\mid X]
=b(X)-\mathbb{E}[A\mid X]\log b(X).
\]
When $\mathbb{E}[A\mid X]>0$, this is minimized at
$b(X)=\mathbb{E}[A\mid X]$; when $\mathbb{E}[A\mid X]=0$, the optimum is
approached as $b(X)\downarrow0$.
Applying this separately to the birth and death rates gives $\lambda_i^\star$ and $\mu_i^\star$. Moreover, using $0\log 0=0$,
\[
\ell(a,b)-\ell(a,a)
=a\log\frac{a}{b}-a+b
=D_{\mathrm{GKL}}(a,b),
\]
so the excess diagonal loss is an integrated generalized KL divergence between the exact and model marginal rates, exactly as in count-FM \citep{wei2026count-flow-matching}.

\subsection{Off-diagonal compositional self-distillation}
\label{app:chapman-kolmogorov-distillation}

For fixed $s<u<t$, define the composed target kernel
\[
R_{\theta;s,u,t}(y\mid x)
=
\sum_z \widetilde K_{\theta,s,u}(z\mid x)
\widetilde K_{\theta,u,t}(y\mid z).
\]
The compositional consistency loss in Equation~\eqref{eq:chapman-kolmogorov-loss}
can be expressed as
\begin{equation}
\begin{aligned}
\mathcal L_{\mathrm{comp}}(\theta)
={}&
\mathbb E_{X_s\sim p_s}\Big[
H(R_{\theta;s,u,t}(\cdot\mid X_s))+
D_{\mathrm{KL}}\!\left(
R_{\theta;s,u,t}(\cdot\mid X_s)
\Vert K_{\theta,s,t}(\cdot\mid X_s)
\right)\Big]\\
={}&
-\mathbb E_{\substack{
X_s\sim p_s,\\
Y\sim R_{\theta;s,u,t}(\cdot\mid X_s)
}}
\log K_{\theta,s,t}(Y\mid X_s).
\end{aligned}
\label{eq:ck-cross-entropy-decomposition}
\end{equation}
With gradients stopped through the composed target,
the entropy term has no gradient. Thus the sampled
cross-entropy used in training has the same gradient
as the expected conditional KL divergence.
Zero expected KL implies compositional consistency
for $p_s$-almost every starting state. More generally, for $M\ge2$
equally spaced subintervals with endpoints $u_j=s+j(t-s)/M$,
we construct the target
$R^{(M)}_{\theta;s,t}=\widetilde K_{\theta,u_0,u_1}\cdots
\widetilde K_{\theta,u_{M-1},u_M}$.
The same cross-entropy objective and entropy--KL decomposition apply,
with $M=2$ recovering the midpoint construction.
% For fixed $x$ and arbitrary $s<u<t$, define the composed target kernel
% \[
% R_{\theta;s,u,t}(y\mid x)
% =
% \sum_z\widetilde K_{\theta,s,u}(z\mid x)
% \widetilde K_{\theta,u,t}(y\mid z),
% \]
% where the target kernels are treated as fixed when optimizing the direct transition. The conditional cross-entropy is
% \[
% \mathcal{C}_\theta(x,s,u,t)
% =
% \mathbb{E}_{Y\sim R_{\theta;s,u,t}}
% [-\log K_{\theta,s,t}(Y\mid x)].
% \]
% Adding and subtracting $\log R_{\theta;s,u,t}(Y\mid x)$ gives
% \begin{equation}
% \mathcal{C}_\theta(x,s,u,t)
% =
% H(R_{\theta;s,u,t})
% +
% D_{\mathrm{KL}}\!\left(
% R_{\theta;s,u,t}
% \Vert K_{\theta,s,t}(\cdot\mid x)
% \right).
% \label{eq:ck-cross-entropy-decomposition}
% \end{equation}
% Thus, with gradients stopped through the target, minimizing the
% cross-entropy is equivalent to minimizing the conditional KL divergence.
% Since $\widetilde K_\theta$ and $K_\theta$ have identical values, zero KL divergence implies the Chapman--Kolmogorov identity.
The same divergence controls the corresponding marginal discrepancy. For any starting distribution $\nu$, data processing followed by Pinsker's inequality gives
\[
D_{\mathrm{TV}}\!\left(
\nu K_{\theta,s,t},
\nu K_{\theta,s,u}K_{\theta,u,t}
\right)
\le
\sqrt{\frac12
\mathbb{E}_{X\sim\nu}
D_{\mathrm{KL}}\!\left(
(K_{\theta,s,u}K_{\theta,u,t})(\cdot\mid X)
\Vert
K_{\theta,s,t}(\cdot\mid X)
\right)}.
\]

\subsection{Identification from midpoint consistency}
\label{app:midpoint-identification}

Let $P^\theta_{s,t}$ be the unique nonexplosive transition family
of the learned generator $Q_t^\theta$. Suppose that, for some finite
$C$ and sufficiently small $h>0$,
\begin{equation}
\sup_{\substack{x\in\mathcal X\\0\le r\le\tau-h}}
D_{\mathrm{TV}}\!\left(K_{\theta,r,r+h}(\cdot\mid x),
P^\theta_{r,r+h}(\cdot\mid x)\right)\le Ch^2,
\label{eq:uniform-short-time-accuracy}
\end{equation}
which is stronger than the pointwise small-time expansion in
Proposition~\ref{prop:diagonal-generator-consistency}.
Assume also that $K_{\theta,t,t}=I$ and, for every $s<t$,
\[
K_{\theta,s,t}=K_{\theta,s,u}K_{\theta,u,t},
\qquad u=(s+t)/2.
\]
Fix an interval of length $H=t-s$. Applying this midpoint identity
repeatedly expresses $K_{\theta,s,t}$ as a product over $2^n$ intervals
of length $h=H/2^n$. By total-variation contraction and
\eqref{eq:uniform-short-time-accuracy}, for sufficiently large $n$,
\[
\sup_xD_{\mathrm{TV}}\!\left(K_{\theta,s,t}(\cdot\mid x),
P^\theta_{s,t}(\cdot\mid x)\right)
\le 2^n Ch^2=\frac{CH^2}{2^n}\longrightarrow 0.
\]
Therefore $K_{\theta,s,t}=P^\theta_{s,t}$ for every $s\le t$.
Consequently, compositional consistency permits composition
over any partition, including $M$ equally spaced subintervals for
any integer $M\ge2$. In particular, if $Q_t^\theta=Q_t^\star$,
these are the target transitions.

Midpoint splitting minimizes the longest subinterval and avoids
nearly trivial endpoint splits, which provides a more balanced target.
However, in deterministic continuous flow maps, neither midpoint nor
uniform sampling consistently dominates across datasets and inference
step counts \citep[Table~1]{boffi2025flow-map-self-distillation}.
Further work should systematically compare splitting schedules
for count-valued stochastic flow maps.
% Therefore $K_{\theta,s,t}=P^\theta_{s,t}$ for every $s\le t$.
% In particular, if $Q_t^\theta=Q_t^\star$, these are the target transitions.

\subsection{Terminal error analysis}
\label{app:terminal-error-analysis}

We separate the terminal error into endpoint truncation, estimation of the diagonal generator, and approximation of its finite-time transition family. Let $p_t^\star$ denote the marginal distribution of the exact signed-binomial bridge, let $Q_t^\star$ denote its marginal generator, and let $Q_t^\theta$ denote the model diagonal generator with exact transition family $P^\theta_{s,t}$.
For a partition $\Pi=\{0=t_0<\cdots<t_L=\tau\}$,
define $K_\theta^\Pi=\prod_{k=0}^{L-1}K_{\theta,t_k,t_{k+1}}$
and the finite-time map error
$\mathcal E_{\mathrm{map}}(\Pi)=
D_{\mathrm{TV}}(p_0P^\theta_{0,\tau},p_0K_\theta^\Pi)$.
By the triangle inequality,
% For a partition $\Pi=\{0=t_0<\cdots<t_L=\tau\}$, define $K_\theta^\Pi=\prod_{k=0}^{L-1}K_{\theta,t_k,t_{k+1}}$. By the triangle inequality,
\[
D_{\mathrm{TV}}(p_1,p_0K_\theta^\Pi)
\le{}
D_{\mathrm{TV}}(p_1,p_\tau^\star)
+
D_{\mathrm{TV}}(p_\tau^\star,p_0P^\theta_{0,\tau}) +
D_{\mathrm{TV}}(p_0P^\theta_{0,\tau},p_0K_\theta^\Pi).
\]
We bound the three terms separately.

\subsubsection{Endpoint truncation}
\label{app:endpoint-truncation-error}

For an endpoint pair $(x_0,x_1)$, let $n_i=|x_{1,i}-x_{0,i}|$. At $\tau=1-\varepsilon$, coordinate $i$ has completed all $n_i$ required unit changes with probability $\tau^{n_i}$. Conditional independence across coordinates gives
\[
\mathbb{P}(X_\tau=X_1\mid X_0=x_0,X_1=x_1)
=\tau^{\|x_1-x_0\|_1}.
\]
Using $(X_\tau,X_1)$ from the same bridge as a coupling,
\begin{equation}
D_{\mathrm{TV}}(p_\tau^\star,p_1)
\le
\mathbb{E}_\pi[1-\tau^{\|X_1-X_0\|_1}]
\le
\varepsilon\,\mathbb{E}_\pi\|X_1-X_0\|_1.
\label{eq:endpoint-truncation-bound}
\end{equation}
This is the endpoint bound from count-FM \citep{wei2026count-flow-matching}.

\subsubsection{Generator estimation error}
\label{app:generator-estimation-error}

Define the integrated rate error
\[
\mathcal{E}_{\mathrm{rate}}(\theta)
=
\int_0^\tau
\mathbb{E}_{X_t\sim p_t^\star}
\sum_{i=1}^d
\left[
D_{\mathrm{GKL}}(\lambda_i^\star(X_t,t),\lambda_{\theta,i}(X_t,t))
+
D_{\mathrm{GKL}}(\mu_i^\star(X_t,t),\mu_{\theta,i}(X_t,t))
\right]dt.
\]
Under the usual absolute-continuity and integrability conditions for path-space relative entropy of jump processes \citep{arampatzis2015pathwise-sensitivity},
\[
D_{\mathrm{KL}}(p_\tau^\star\Vert p_0P^\theta_{0,\tau})
\le\mathcal{E}_{\mathrm{rate}}(\theta).
\]
Pinsker's inequality and uniform sampling of $t$ in Equation~\eqref{eq:diagonal-rate-loss} then give
\begin{equation}
D_{\mathrm{TV}}(p_\tau^\star,p_0P^\theta_{0,\tau})
\le
\sqrt{\frac{\tau}{2}
\left[
\mathcal{L}_{\mathrm{diag}}(\theta)-\mathcal{L}_{\mathrm{diag}}^\star
\right]}.
\label{eq:generator-estimation-bound}
\end{equation}
This is the same rate-estimation argument used in count-FM \citep{wei2026count-flow-matching}.

\subsubsection{Finite-time Flow Map error}
\label{app:finite-time-map-error}
We now bound the third term of the terminal error decomposition.
For $u=(s+t)/2$, define the midpoint inconsistency
\[
\rho(s,t)=\sup_x D_{\mathrm{TV}}\!\left(
K_{\theta,s,t}(\cdot\mid x),
(K_{\theta,s,u}K_{\theta,u,t})(\cdot\mid x)\right).
\]
Assume that its normalized worst-case value is finite:
\begin{equation}
\eta_{\mathrm{comp}}=\sup_{0\le s<t\le\tau}
\frac{\rho(s,t)}{(t-s)^2}<\infty.
\label{eq:normalized-midpoint-residual}
\end{equation}
Thus, splitting an interval of length $H$ once changes its learned
kernel by at most $\eta_{\mathrm{comp}}H^2$ in worst-case TV.
Repeated bisection contributes at most
$\eta_{\mathrm{comp}}H^2(1+\tfrac12+\tfrac14+\cdots)$.
More precisely, after $n$ subdivisions, each remaining learned
transition has length $h=H/2^n$. By
\eqref{eq:uniform-short-time-accuracy}, its error relative to the
exact CTMC transition is at most $Ch^2$, where $C$ is the uniform
local-error constant. TV contraction then gives
\begin{equation}
\sup_x D_{\mathrm{TV}}\!\left(K_{\theta,s,t}(\cdot\mid x),
P^\theta_{s,t}(\cdot\mid x)\right)
\leq \eta_{\mathrm{comp}}H^2
\sum_{j=0}^{n-1}2^{-j}+\frac{CH^2}{2^n}
\xrightarrow[n\to\infty]{}2\eta_{\mathrm{comp}}H^2.
\label{eq:midpoint-error-refinement}
\end{equation}

For an inference partition $\Pi$, TV contraction under kernel
composition therefore yields
\begin{equation}
\mathcal E_{\mathrm{map}}(\Pi)
\le\sum_{k=0}^{L-1}\sup_x D_{\mathrm{TV}}\!\left(
K_{\theta,t_k,t_{k+1}}(\cdot\mid x),
P^\theta_{t_k,t_{k+1}}(\cdot\mid x)\right)
\le 2\eta_{\mathrm{comp}}\sum_{k=0}^{L-1}(t_{k+1}-t_k)^2.
\label{eq:finite-time-map-bound}
\end{equation}
For $L$ equally spaced inference steps, the time increment is
$\tau/L$, so
\begin{equation}
\mathcal E_{\mathrm{map}}(\Pi)\le
\frac{2\eta_{\mathrm{comp}}\tau^2}{L}.
\label{eq:nfe-map-error-bound}
\end{equation}

\subsubsection{Terminal distribution error bound}
\label{app:combined-terminal-bound}
Combining the endpoint bound~\eqref{eq:endpoint-truncation-bound},
the rate-estimation bound~\eqref{eq:generator-estimation-bound},
and the map bound~\eqref{eq:nfe-map-error-bound} gives,
for a uniform $L$-step partition, under the absolute-continuity,
integrability, and uniform-accuracy assumptions above,
\begin{equation}
\begin{aligned}
D_{\mathrm{TV}}(p_1,p_0K_\theta^\Pi)
\le{}& \varepsilon\,\mathbb E_\pi\|X_1-X_0\|_1
+\sqrt{\frac{\tau}{2}\big[\mathcal L_{\mathrm{diag}}(\theta)
-\mathcal L_{\mathrm{diag}}^\star\big]}
+\frac{2\eta_{\mathrm{comp}}\tau^2}{L}.
\end{aligned}
\label{eq:explicit-terminal-error-bound}
\end{equation}
This is the same decomposition as Equation~\eqref{eq:terminal-error-bound},
with an explicit finite-time approximation term.

\section{Additional Simulation Details and Results}
\label{app:simulation}

\subsection{Experimental details}
\label{app:simulation-details}

All experiments use five random seeds and evaluate
NFE $\in\{1,2,4,16,64,128,256\}$. Count Flow Map, Count-FM, CountsDiff,
D3PM, and Discrete Flow Maps are trained for 10,000 parameter updates using
MLPs with hidden width 256 and depth 4. The effective batch size is 512 and
the learning rate is $10^{-3}$. The two Count-FM samplers use the same trained
rate network and differ only at inference. Count Flow Map uses eight mixture components in its finite-time
transition head and constructs its compositional self-distillation
target by composing eight teacher transitions over equally spaced
subintervals. Discrete Flow Maps uses 8,000
diagonal updates followed by 2,000 self-distillation updates, and D3PM uses
256 native diffusion steps. Count Flow Map and Count-FM use the signed-binomial
bridge with $\tau=0.995$. For D3PM and Discrete Flow Maps, counts are represented
by categories $0,\ldots,C_{\max}$ plus an overflow category.

For the 2-D simulation, we use 30,000 generated and target samples per seed.
Total variation (TV) distance is computed against the analytic target PMF,
while empirical $W_2$ is computed by optimal matching on 1,500 samples and
averaged over three subsamples. For the 32-D simulation, we use 5,000
generated and target samples per seed. Sliced $W_2$
\citep{bonneel2015sliced-wasserstein} uses 128 random projections, and
MMD$^2_{\mathrm{RBF}}$ \citep{gretton2012kernel-test} uses an RBF kernel
with bandwidth selected by the median-distance heuristic. Generation time is the median of three runs after
one warm-up.

\subsection{2-D simulation results}
\label{app:simulation-2d-results}

The 2-D target is an equal-weight two-component Gamma--Poisson mixture with
component mean vectors $(60,5)$ and $(60,40)$ and concentration vectors
$(160,80)$ and $(160,140)$, respectively. Table~\ref{tab:sim-2d-app} reports selected operating points.

\begin{table}[t]
\caption{2-D simulation endpoint quality at selected inference budgets.
Entries are mean $\pm$ standard deviation over five seeds.
Generation times are for batches of 30,000 samples.
Lower is better for all metrics. The lowest mean quality metric at each NFE is shown in bold.}
\label{tab:sim-2d-app}
\centering
\begin{tabular}{llccc}
\toprule
Method & NFE & TV$\downarrow$ & $W_2\downarrow$ & Time (s)$\downarrow$ \\
\midrule
Count Flow Map & 1 & \textbf{0.131$\pm$0.004} & \textbf{2.47$\pm$0.34} & 0.011$\pm$0.002 \\
 & 4 & \textbf{0.116$\pm$0.003} & 2.77$\pm$0.45 & 0.043$\pm$0.008 \\
 & 128 & 0.103$\pm$0.002 & 2.64$\pm$0.61 & 1.351$\pm$0.255 \\
 & 256 & \textbf{0.103$\pm$0.001} & 2.26$\pm$0.34 & 2.702$\pm$0.513 \\
\midrule
Count-FM + unit jump & 1 & 0.742$\pm$0.002 & 30.10$\pm$0.34 & 0.006$\pm$0.001 \\
 & 4 & 0.698$\pm$0.001 & 26.99$\pm$0.43 & 0.021$\pm$0.002 \\
 & 128 & 0.199$\pm$0.002 & 4.84$\pm$0.22 & 0.656$\pm$0.085 \\
 & 256 & 0.145$\pm$0.002 & 3.16$\pm$0.28 & 1.397$\pm$0.204 \\
\midrule
Count-FM + binomial $\tau$-leap & 1 & 0.929$\pm$0.005 & 15.21$\pm$0.06 & 0.006$\pm$0.001 \\
 & 4 & 0.295$\pm$0.005 & 4.84$\pm$0.09 & 0.022$\pm$0.002 \\
 & 128 & 0.104$\pm$0.002 & 2.47$\pm$0.32 & 0.694$\pm$0.100 \\
 & 256 & 0.104$\pm$0.002 & \textbf{2.04$\pm$0.37} & 1.397$\pm$0.191 \\
\midrule
CountsDiff & 1 & 1.000$\pm$0.000 & 21.74$\pm$0.19 & 0.006$\pm$0.001 \\
 & 4 & 0.430$\pm$0.007 & 7.60$\pm$0.45 & 0.019$\pm$0.004 \\
 & 128 & \textbf{0.102$\pm$0.003} & \textbf{2.17$\pm$0.44} & 0.594$\pm$0.125 \\
 & 256 & 0.104$\pm$0.003 & 2.26$\pm$0.23 & 1.187$\pm$0.250 \\
\midrule
D3PM & 1 & 0.989$\pm$0.004 & 27.20$\pm$0.31 & 0.006$\pm$0.001 \\
 & 4 & 0.978$\pm$0.007 & 26.74$\pm$0.31 & 0.028$\pm$0.003 \\
 & 128 & 0.965$\pm$0.013 & 26.37$\pm$0.47 & 0.923$\pm$0.098 \\
 & 256 & 0.184$\pm$0.008 & 2.57$\pm$0.52 & 1.850$\pm$0.194 \\
\midrule
Discrete Flow Maps & 1 & 0.168$\pm$0.003 & 3.77$\pm$0.65 & 0.006$\pm$0.001 \\
 & 4 & 0.132$\pm$0.004 & \textbf{2.69$\pm$0.86} & 0.022$\pm$0.002 \\
 & 128 & 0.146$\pm$0.002 & 4.69$\pm$0.26 & 0.704$\pm$0.079 \\
 & 256 & 0.148$\pm$0.003 & 5.05$\pm$0.33 & 1.407$\pm$0.160 \\
\bottomrule
\end{tabular}
\end{table}

\subsection{2-D endpoint and path visualizations}
\label{app:simulation-2d-visuals}

Figure~\ref{fig:sim-2d-endpoints-app} shows endpoint samples at the four
operating points reported in Table~\ref{tab:sim-2d-app}. At NFE=1, Count Flow
Map already closely recovers the two target modes. Both Count-FM samplers,
CountsDiff, and D3PM remain visibly far from the target, while Discrete Flow
Maps captures the two-mode geometry but has a larger $W_2$ error than Count
Flow Map. As NFE increases, the iterative methods progressively approach the
target. Notably,
D3PM remains far from the target through NFE=128 and only approaches it at
NFE=256, likely reflecting its less smooth categorical transition path and the
resulting difficulty of coarse inference.

\begin{figure}[t]
    \centering
    \includegraphics[width=\linewidth]{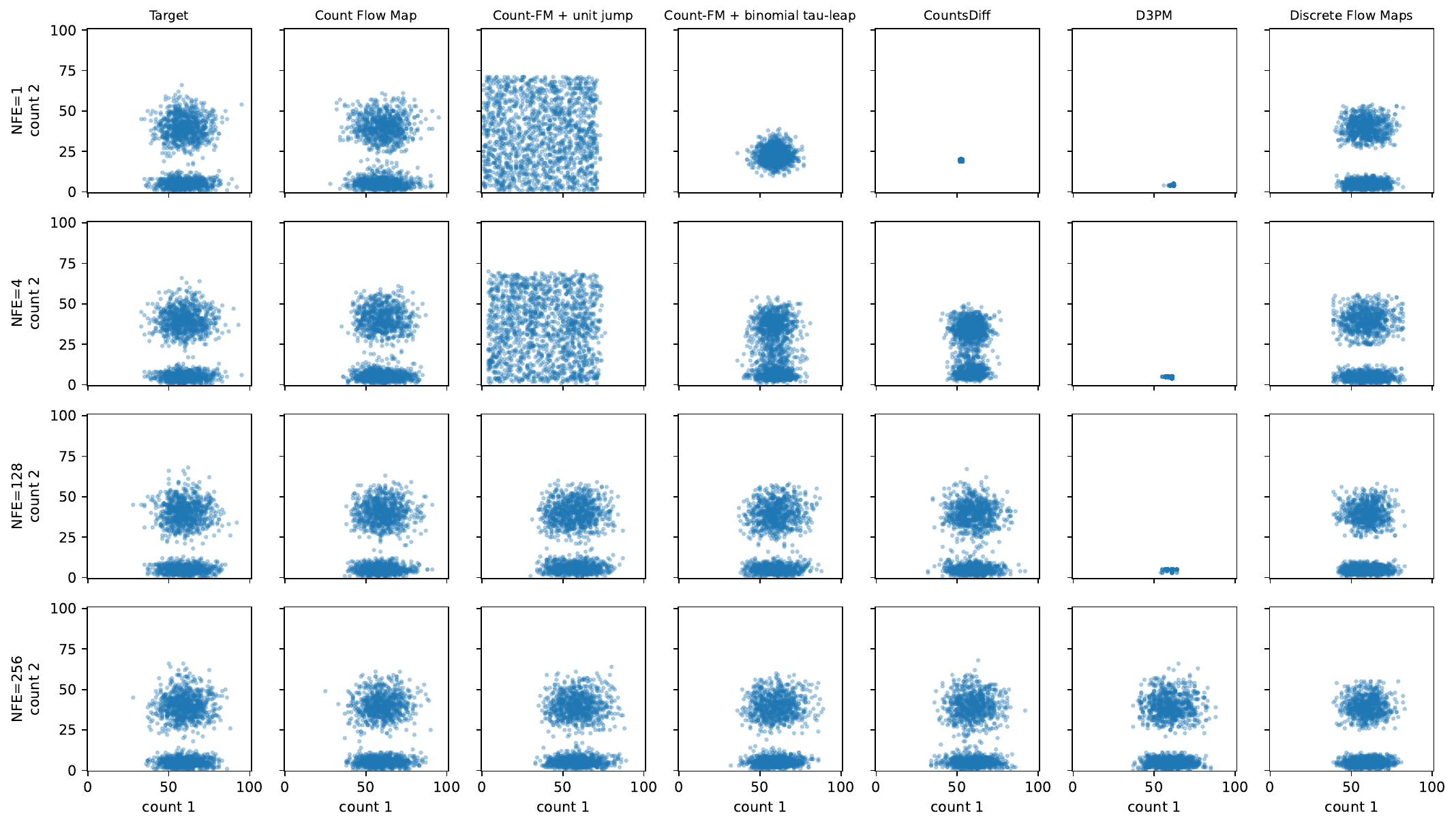}
    \caption{2-D endpoint samples at NFE $\in\{1,4,128,256\}$. Columns show the target and the six compared methods, and rows correspond to different inference budgets.}
    \label{fig:sim-2d-endpoints-app}
\end{figure}

Figure~\ref{fig:sim-2d-paths-app} shows generation paths at normalized progress
$q\in\{0,0.25,0.5,0.75,1\}$. For Count Flow Map, the five columns correspond
to the initial distribution and the distributions after each of four finite-time
map evaluations over $[0,\tau]$. Discrete Flow Maps is also shown using four
finite-time compositions, while Count-FM, CountsDiff, and D3PM are shown at
matched progress along 256-step sampling paths. Count Flow Map shows a smooth progression through count space,
similar to Count-FM, whereas the categorical-state methods exhibit more abrupt
transitions, reflecting the mismatch between categorical representations and
the ordered geometry of count data.

\begin{figure}[t]
    \centering
    \includegraphics[width=0.98\linewidth]{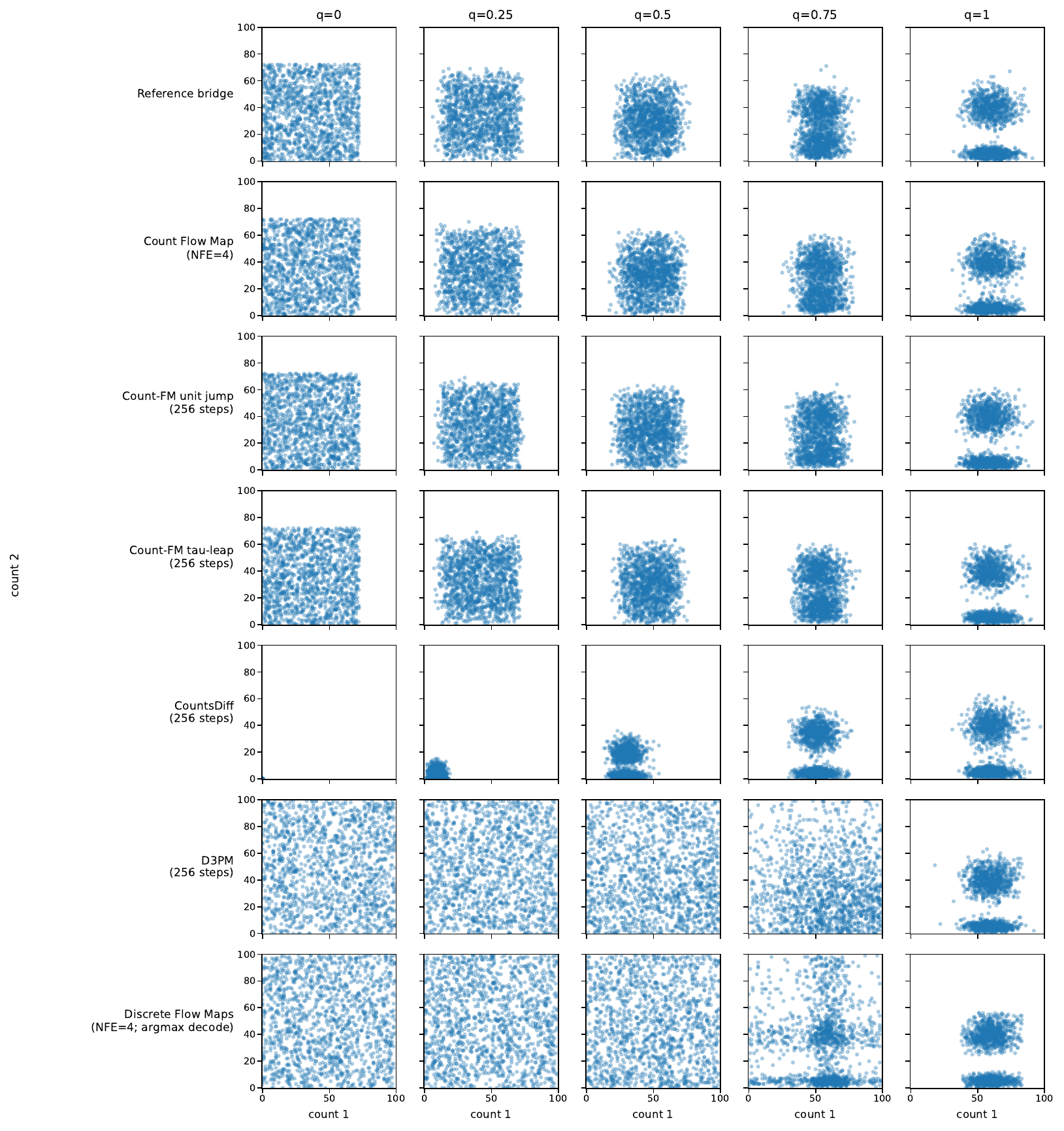}
    \caption{Generation paths for the 2-D simulation at normalized progress $q$. Count Flow Map and Discrete Flow Maps use NFE=4 finite-time compositions. Count-FM, CountsDiff, and D3PM are shown at matched progress along 256-step sampling paths.}
    \label{fig:sim-2d-paths-app}
\end{figure}

\subsection{32-D simulation}
\label{app:simulation-32d}

We consider a 32-dimensional Gamma--Poisson factor mixture with three mixture
components. Within each component, four latent Gamma factors induce dependence
across the 32 count variables. The resulting distribution is multimodal and
overdispersed, with a nominal count scale of 20. For Count Flow Map and
Count-FM, the source is uniform on $\{0,\ldots,50\}^{32}$. We evaluate sample
quality using sliced $W_2$ and RBF MMD$^2$. Table~\ref{tab:sim-32d-app} reports selected
operating points.

\begin{table}[t]
\caption{32-D simulation endpoint quality at selected inference budgets.
Entries are mean $\pm$ standard deviation over five seeds.
Generation times are for batches of 5,000 samples.
Lower is better for all metrics. The lowest mean quality metric at each NFE is shown in bold.}
\label{tab:sim-32d-app}
\centering
\begin{tabular}{llccc}
\toprule
Method & NFE & SW$_2\downarrow$ & MMD$^2_{\rm RBF}\downarrow$ & Time (s)$\downarrow$ \\
\midrule
Count Flow Map & 1 & \textbf{0.533$\pm$0.132} & \textbf{0.0005$\pm$0.0005} & 0.006$\pm$0.001 \\
 & 4 & \textbf{0.621$\pm$0.119} & \textbf{0.0008$\pm$0.0004} & 0.018$\pm$0.004 \\
 & 128 & \textbf{0.442$\pm$0.076} & \textbf{0.0002$\pm$0.0001} & 0.537$\pm$0.144 \\
 & 256 & \textbf{0.484$\pm$0.108} & \textbf{0.0002$\pm$0.0001} & 1.073$\pm$0.290 \\
\midrule
Count-FM + unit jump & 1 & 10.898$\pm$0.489 & 0.1716$\pm$0.0029 & 0.002$\pm$0.000 \\
 & 4 & 9.395$\pm$0.281 & 0.1762$\pm$0.0018 & 0.006$\pm$0.001 \\
 & 128 & 1.150$\pm$0.079 & 0.0060$\pm$0.0009 & 0.159$\pm$0.022 \\
 & 256 & 0.699$\pm$0.078 & 0.0015$\pm$0.0002 & 0.317$\pm$0.045 \\
\midrule
Count-FM + binomial $\tau$-leap & 1 & 4.520$\pm$0.156 & 0.1028$\pm$0.0042 & 0.003$\pm$0.000 \\
 & 4 & 2.557$\pm$0.111 & 0.0232$\pm$0.0008 & 0.006$\pm$0.001 \\
 & 128 & 0.813$\pm$0.096 & 0.0010$\pm$0.0001 & 0.156$\pm$0.022 \\
 & 256 & 0.834$\pm$0.090 & 0.0009$\pm$0.0003 & 0.313$\pm$0.046 \\
\midrule
CountsDiff & 1 & 7.577$\pm$0.111 & 0.2551$\pm$0.0134 & 0.003$\pm$0.001 \\
 & 4 & 4.328$\pm$0.135 & 0.0865$\pm$0.0022 & 0.009$\pm$0.004 \\
 & 128 & 0.622$\pm$0.067 & 0.0006$\pm$0.0001 & 0.246$\pm$0.105 \\
 & 256 & 0.599$\pm$0.154 & 0.0004$\pm$0.0001 & 0.491$\pm$0.209 \\
\midrule
D3PM & 1 & 7.681$\pm$0.329 & 0.2654$\pm$0.0131 & 0.007$\pm$0.001 \\
 & 4 & 4.674$\pm$0.094 & 0.0910$\pm$0.0045 & 0.038$\pm$0.006 \\
 & 128 & 2.545$\pm$0.107 & 0.0196$\pm$0.0007 & 1.296$\pm$0.209 \\
 & 256 & 1.476$\pm$0.045 & 0.0053$\pm$0.0003 & 2.604$\pm$0.418 \\
\midrule
Discrete Flow Maps & 1 & 4.636$\pm$0.078 & 0.0845$\pm$0.0022 & 0.008$\pm$0.001 \\
 & 4 & 2.170$\pm$0.144 & 0.0124$\pm$0.0006 & 0.026$\pm$0.005 \\
 & 128 & 1.320$\pm$0.138 & 0.0028$\pm$0.0004 & 0.783$\pm$0.117 \\
 & 256 & 1.661$\pm$0.263 & 0.0037$\pm$0.0007 & 1.563$\pm$0.233 \\
\bottomrule
\end{tabular}
\end{table}

Count Flow Map achieves the lowest mean sliced $W_2$ and
MMD$^2_{\mathrm{RBF}}$ at every evaluated NFE.
Its mean sliced $W_2$ ranges from $0.442$ to $0.621$ over the full grid,
with strong quality already at NFE=1. Figure~\ref{fig:sim-32d-efficiency-app}
shows the quality--efficiency curves.

\begin{figure}[t]
    \centering
    \includegraphics[width=\linewidth]{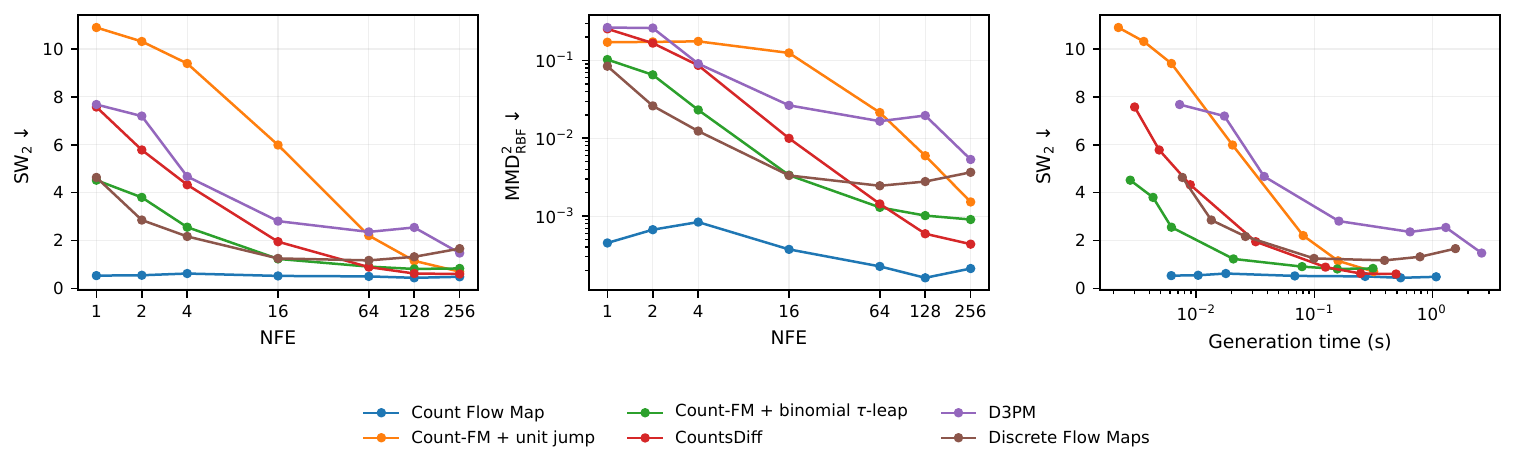}
    \caption{32-D quality--efficiency curves. Left and center show sliced $W_2$ and RBF MMD$^2$ versus NFE, where MMD$^2$ uses a logarithmic vertical scale. Right shows sliced $W_2$ versus measured generation time for 5,000 samples. Points show means over five seeds. Count Flow Map achieves the lowest mean sliced $W_2$ and RBF MMD$^2$ across the full NFE grid.}
    \label{fig:sim-32d-efficiency-app}
\end{figure}

\section{Additional Single-Cell Results}
\label{app:scrna}

\subsection{Data, split, and evaluation}
\label{app:scrna-details}

We construct a Tahoe-100M panel containing all combinations of five
cell lines, eight drugs, and three doses ($0.05$, $0.5$, and $5$),
giving 120 conditions $c=(\ell,d,a)$. The cell-line identifiers are
\texttt{CVCL\_0023}, \texttt{CVCL\_0028}, \texttt{CVCL\_0069},
\texttt{CVCL\_0099}, and \texttt{CVCL\_0131}. The drugs are
Berbamine (dihydrochloride), Binimetinib, Cytarabine, Daidzin,
Everolimus, Flumatinib (mesylate), Fumaric acid, and Gemcitabine.
Each condition contains multiple measured cells; preparation retains
up to 600 treated cells per condition, yielding 70,064 treated cells
in total. These counts refer to individual cells, not independent
experimental replicates.

The split is performed at the condition level: all retained treated
cells sharing the same cell-line--drug--dose combination belong to
exactly one split. Using the fixed split seed 42, a coverage-constrained
assignment selects 24 test conditions (20\%) and then 18 validation
conditions (15\%), leaving 78 training conditions (65\%). Selection
spreads holdouts across cell-line--drug pairs, cell lines, and drugs.
A condition can be held out only if at least one other dose of its
cell-line--drug pair remains in training. The resulting split is:

\begin{center}
\small
\begin{tabular}{lrr rrr}
\toprule
& & & \multicolumn{3}{c}{Conditions at each dose} \\
\cmidrule(lr){4-6}
Split & Conditions & Treated cells & $0.05$ & $0.5$ & $5$ \\
\midrule
Training   & 78  & 45,024 & 28 & 25 & 25 \\
Validation & 18  & 10,710 &  4 &  6 &  8 \\
Testing    & 24  & 14,330 &  8 &  9 &  7 \\
\midrule
Total      & 120 & 70,064 & 40 & 40 & 40 \\
\bottomrule
\end{tabular}
\end{center}

Training covers all five cell lines, eight drugs, three doses,
and 40 combinations of cell line and drug. Held-out conditions
evaluate dose interpolation and extrapolation for these known
combinations. All methods, couplings, and training seeds use
the same split.

Cells are represented by counts for 2,000 highly variable genes selected using training treated cells and their matched DMSO
controls. The 14-dimensional context combines five cell-line indicators, eight drug indicators, and standardized
$\log(1+\text{dose})$. Dose-standardization statistics are computed from training
conditions only.

Both count-flow models use independent and OT training couplings.
Independent coupling samples DMSO controls with replacement,
matched by cell line and plate, allowing control reuse across
conditions and splits. OT additionally permutes treated endpoints
within each condition and plate to minimize total raw-count
$L_1$ distance while preserving both empirical marginals.
These numerical couplings do not establish biological cell
correspondences. Validation and test data are identical across
couplings.

Evaluation uses the same fixed subset of 400 cells per held-out
condition for every method and seed: 7,200 validation cells and
9,600 test cells, drawn from the larger retained pools above.
For distributional evaluation, counts are library-size normalized
to $10^4$, log-transformed, and projected onto 50 principal
components fitted using training cells only. This projection is
used for evaluation; generation remains in the original
2,000-dimensional count space. Sliced $W_2$ uses 128 random
projections. Metrics are computed within each condition and averaged
equally across conditions. For each trained method, we report the
mean and sample standard deviation of these averages across
training seeds 42, 123, and 2026 on the fixed split. This describes
variation across training runs, not biological replicate uncertainty.
Linear and Sinkhorn baselines are each evaluated once and have no
across-training-seed standard deviation.

For the table, perturbation effects relative to matched DMSO are
\[
\Delta_{cg}^{\log\mathrm{FC}}
=\log_2\frac{\bar x_{1,cg}+0.1}{\bar x_{0,cg}+0.1}.
\]
Generated effects replace $\bar x_{1,cg}$ by the generated mean.
LogFC Pearson correlation is computed on the 200 genes with largest observed $|\Delta_{cg}^{\log\mathrm{FC}}|$; top-gene overlap is the fraction of genes
shared by the observed and generated top-200 sets. These response genes
are ranked by effect magnitude, without a significance-test threshold.

\paragraph{Neural architectures.}
The main MLP blocks of the count models, conditional NB-VAE,
CPA, and scGen/scVIDR use three hidden layers of 512 units with
SiLU activations and no batch normalization or dropout.
The VAE-based baselines use latent dimension 128.
Method-specific auxiliary networks and output heads are retained.

\paragraph{Count-model training and selection.}
Count Flow Map and Count-FM use 50,000 updates with batch size 64. Count Flow Map uses four mixture components and an eight-transition compositional target over equal subintervals, with final consistency weight one. Generation ends at $\tau=0.995$.

Every 5,000 updates, EMA checkpoints are evaluated by
condition-averaged validation sliced $W_2$, using 1 NFE for
Count Flow Map and 128 binomial $\tau$-leap steps for Count-FM.
The best checkpoint for each model, coupling, and seed is fixed
across inference budgets. Both Count-FM samplers share the
selected rate network and select NFE separately on validation
from $\{16,64,128,256,512\}$ for unit jumps and
$\{1,4,16,64,128,256\}$ for binomial $\tau$-leaps.
Count Flow Map is reported at fixed 1 and 16 NFE, with its
full curve evaluated on the latter grid.

\paragraph{Baseline training and selection.}
Conditional NB-VAE and CPA use 50,000 updates with batch sizes 64 and 256, respectively, and validation every 5,000 updates. The shared scGen/scVIDR VAE uses batch size 64 for 100 epochs totaling 140,700 updates, with validation every 10 epochs. scGen uses nearest-dose latent arithmetic and scVIDR uses continuous log-dose scaling.

Neural baselines select their best checkpoint from ten
evaluations of condition-averaged validation sliced $W_2$,
using the same validation cells and PCA as the count models.
scGen and scVIDR share a VAE training trajectory per seed but
select checkpoints separately using their downstream predictions.
All neural models complete their prescribed training budgets.
Parameter fitting and response-statistic estimation use training
data only, with validation used for selection and test data
reserved for final evaluation.

\paragraph{Generation timing.}
Generation time is measured on the full 9,600-cell test subset,
using batches of 256 where applicable. Each seed's time is the
median of three runs after one untimed warm-up, including
generation transfers and excluding model loading, response-model
preparation, PCA, and metric computation. Training-only response
statistics for scGen/scVIDR and library-size summaries for CPA
are prepared before timing. The table summarizes these times
across training seeds. Count-FM's selected NFE can differ between
seeds.

\subsection{Condition-level perturbation recovery}
\label{app:scrna-examples}

We examine perturbation strength and response patterns across genes.
For condition $c$ and gene $g$, let $\bar x_{0,cg}$,
$\bar x_{1,cg}$, and $\bar x_{\mathrm{gen},cg}$ denote mean
raw counts in matched DMSO controls, observed treated cells,
and generated cells, respectively. We define
\[
\begin{aligned}
\delta_{cg}
&=\log(1+\bar x_{1,cg})-\log(1+\bar x_{0,cg}),\\
\widehat\delta_{cg}
&=\log(1+\bar x_{\mathrm{gen},cg})-\log(1+\bar x_{0,cg}).
\end{aligned}
\]
These effects use all 2,000 retained genes and differ from the
table's top-200-gene logFC metric.
Panel C of Figure~\ref{fig:scrna-efficiency} compares observed
strength $\|\delta_c\|_2$ with generated strength
$\|\widehat\delta_c\|_2$ across the 24 test conditions.
We report Pearson correlation across these magnitudes and
fit generated magnitudes against observed magnitudes through
the origin. A slope below one indicates overall underestimation.

The top panels of Figure~\ref{fig:scrna-all-heldout-effects}
project observed and generated gene-response patterns onto
common PCA axes fitted using training-condition effects.
The bottom panels report
$\operatorname{Corr}_g(\delta_{cg},\widehat\delta_{cg})$
at 1 and 16 NFE, with test conditions shown in the same order.
All effect illustrations use one training run per coupling, with the same OT model evaluated at 1 and 16 NFE. The plots show response-pattern recovery across multiple conditions, alongside varying errors in projected direction and magnitude. 
Consistent with Figure~\ref{fig:scrna-efficiency}C, these results
illustrate that Count Flow Map captures variation in perturbation
responses with very few model evaluations.

\begin{figure}[t]
    \centering
    \includegraphics[width=\linewidth]{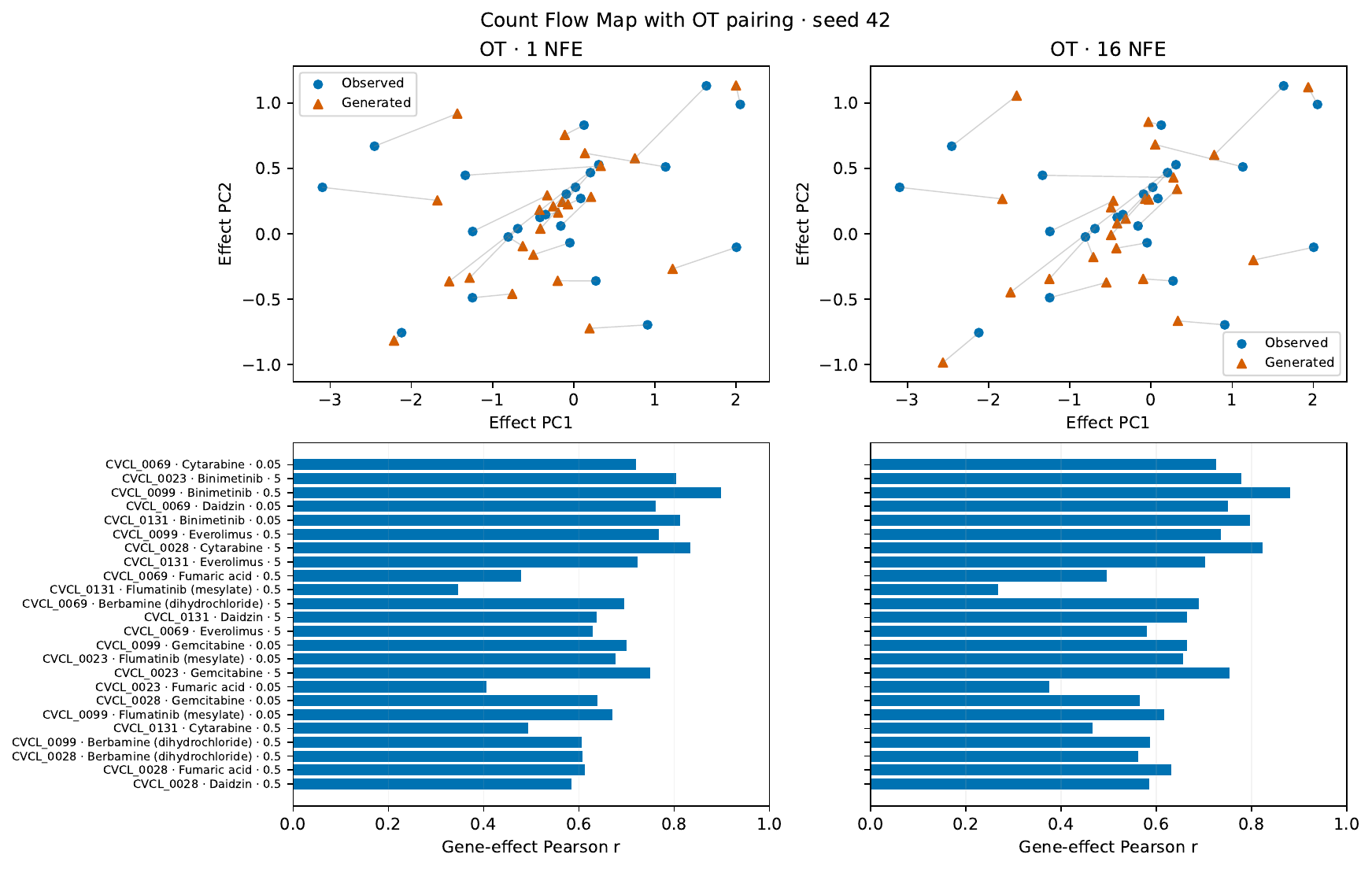}
    \caption{Condition-level perturbation recovery with OT coupling. The left and right columns show 1 and 16 NFE, respectively, using the same model from one fixed training run. \textbf{Top}. observed and generated gene effects projected onto common training-only PCA axes, with lines connecting the same condition. \textbf{Bottom}. Pearson correlations between observed and generated effects across all 2,000 retained genes for each of the 24 test conditions, ordered by decreasing observed effect magnitude. These correlations differ from the table's top-200-gene logFC correlation.}
    \label{fig:scrna-all-heldout-effects}
\end{figure}

\section{Additional Neural Forecasting Details and Results}
\label{app:neural}

\subsection{Data and Forecasting Protocol}
\label{app:neural-data}

We use the 50\,ms-binned SpikeProphecy release of the Steinmetz
recordings \citep{steinmetz2019distributed,spikeprophecy2026data},
pinned to dataset revision
\texttt{10da560900af2b7a5475faca63cbce4be8f860b4}.
The six recording indices are 004, 007, 011, 018, 025, and 033,
containing 703, 1,104, 613, 490, 677, and 538 neurons,
respectively. Their respective lengths are 60,887, 65,200,
68,927, 64,586, 60,569, and 64,850 bins, giving 385,019
population count vectors in total.

We retain the release's chronological 70\%/15\%/15\%
train/validation/test boundaries. Each input consists of ten
consecutive bins, followed by one target bin.
Histories and targets must lie entirely within the same split,
giving 269,451 eligible training windows across the six
recordings. These overlapping windows are not independent
biological replicates. Models are fitted separately to each
recording. Neither neurons nor count vectors are aligned
across recordings. The task conditions only on spike-count
history, without behavioral or stimulus covariates.

Training pairs are the naturally observed consecutive-bin pairs
$(X_0,X_1)=(Y_b,Y_{b+1})$, with context
$H_b=(Y_{b-9},\ldots,Y_b)$.
We use the signed-binomial bridge and conditional kernels
described in the main text, with context fixed throughout each
flow trajectory.
Inference starts from the last observed count vector $Y_b$
and applies a uniform flow-time grid on $[0,\tau]$, with
$\tau=0.995$.
NFE counts kernel or rate-network evaluations within one
predicted physical bin. History encoding is performed once per
forecast. Count-FM's unit-jump and binomial $\tau$-leap
samplers use the same trained rate network.

\paragraph{Architecture and training.}
Both count-flow models use a two-layer GRU history encoder
with width 256 and MLP heads with three hidden layers of
width 384 and SiLU activations.
Count Flow Map uses four mixture components.
% Count features concatenate $x/4$ and
% $\log(1+x)/\log 5$. These are network inputs, while the
% stochastic state and generated outputs remain nonnegative
% integer counts.

For each seed and learning-rate candidate, both count-flow models complete 30,000 updates with batch size 128. Validation and evaluation use EMA weights.
Count Flow Map's diagonal and consistency losses are normalized by the number of neurons, with consistency weight one. Targets compose eight transitions over equal subintervals using the current model weights with stop-gradient. Count-FM uses the same diagonal objective and optimizer settings.

Every 3,000 updates, we evaluate 512 validation histories with
32 draws each. Checkpoints and learning rates are selected by
validation energy score, averaged over 1 and 16 NFE for Count
Flow Map and evaluated at 64 binomial $\tau$-leap steps for
Count-FM. Selected checkpoints are fixed across inference
budgets, with both Count-FM samplers sharing the selected rate
network.

\paragraph{Direct forecasting baselines.}
The direct count mixture uses the same GRU history encoder and
a four-component Poisson--Binomial residual distribution around
the last observed count vector.
Transformer-NB uses a causal Transformer history encoder with
conditionally independent negative-binomial outputs.
The population Poisson GLM predicts conditionally independent
Poisson counts with log rates linear in the standardized
10-bin population history.
These baselines are trained by conditional negative
log-likelihood, with an additional $L_2$ penalty for the GLM,
and selected by validation energy score under the same data
splits, seeds, and training budget.

\subsection{Metrics, Aggregation, and Timing}
\label{app:neural-metrics}

We evaluate 4,096 evenly spaced eligible test histories per
recording, using the same histories for all methods, budgets,
and seeds. This gives 24,576 distinct test histories across
the six recordings.
Repeating training seeds does not create additional test
observations.
Each history has one observed next-bin count vector
$y\in\mathbb N_0^D$ and $M=64$ generated vectors
$x^{(1)},\ldots,x^{(M)}$.
We report the energy score and population CRPS
\citep{gneiting2007proper}. For CRPS, we use the finite-ensemble correction that excludes self-pairs \citep{ferro2014fair}.

The dimension-normalized energy-score estimator is
\begin{equation}
\widehat{\mathrm{ES}}
=
\frac{1}{\sqrt D}
\left[
\frac1M\sum_{m=1}^M\|x^{(m)}-y\|_2
-
\frac1M\sum_{m=1}^{M/2}
\|x^{(m)}-x^{(m+M/2)}\|_2
\right].
\label{eq:neural-energy}
\end{equation}
The second term uses disjoint pairs of independent draws.
Define mean population counts
$a^{(m)}=D^{-1}\sum_jx_j^{(m)}$
and $a_y=D^{-1}\sum_jy_j$.
Their CRPS estimator is
\begin{equation}
\widehat{\mathrm{CRPS}}
=
\frac1M\sum_m|a^{(m)}-a_y|
-
\frac{1}{2M(M-1)}
\sum_{m\ne n}|a^{(m)}-a^{(n)}|.
\label{eq:neural-crps}
\end{equation}

We also report neuron-wise RMSE of the predictive mean,
subtracting the Monte Carlo variance of that mean before
taking the square root.
For a recording with $C$ evaluated histories,
\begin{equation}
\widehat{\mathrm{RMSE}}
=
\sqrt{
\max\left\{
0,
\frac1{CD}\sum_{i=1}^C\sum_{j=1}^D
\left[
(\bar x_{ij}-y_{ij})^2-\frac{s_{ij}^2}{M}
\right]
\right\}
},
\label{eq:neural-rmse}
\end{equation}
where $\bar x_{ij}$ and $s_{ij}^2$ are the sample mean and
unbiased sample variance across generated draws.
The correction removes finite-draw bias from the squared-error
estimate, not from its square root.

Metrics are computed separately for each recording and averaged
equally over the six recordings within each seed.
We report the mean and sample standard deviation of the three
seed-level averages, reflecting training variability on fixed
recordings.

All NFE rows are timed under a common protocol on the
NVIDIA GeForce RTX 2080 Ti used throughout the paper.
Timing processes 16 test histories sequentially, generating
64 draws for each.
After an untimed warm-up pass, we take the median per-history
latency over three repetitions, synchronizing CUDA before
and after each timed pass.
Timing includes history encoding, count sampling, and
input/output device transfers, and excludes model loading
and metric computation.
Thus latency refers to a 64-draw population forecast,
not one generated vector.
Table~\ref{tab:neural-full} reports energy score, population
CRPS, RMSE, and generation time for all 25 count-flow
method--budget combinations, covering 1--64 NFE for Count Flow
Map and 1--256 NFE for Count-FM.
Table~\ref{tab:neural-main} additionally compares the direct
forecasting baselines.

\subsection{Population Events and Forecast Diagnostics}
\label{app:neural-diagnostics}
\paragraph{High-activity events and the dependence control.}
For each recording, let $q$ be the empirical 95th percentile
of total population count over training bins.
We predict the event $\sum_jY_{b+1,j}>q$ using
\[
\widehat p_b
=
\frac1M\sum_{m=1}^M
\mathbf 1\!\left\{
\sum_jx_j^{(m)}>q
\right\}.
\]
We evaluate the binary Brier score
\[
\left(
\widehat p_b-
\mathbf 1\!\left\{\sum_jy_j>q\right\}
\right)^2
\]
\citep{brier1950verification}, averaged over test histories.
Thresholds for recordings 004, 007, 011, 018, 025, and 033
are respectively 220, 437, 226, 273, 183, and 174 total
spikes per bin.
The corresponding observed test-event frequencies are
approximately 4.10\%, 9.18\%, 10.57\%, 3.10\%, 0.90\%,
and 3.93\%.
The event threshold is not recomputed on the test split.

The shuffled control independently permutes draw indices for
every neuron within each conditioning history.
This preserves each neuron's finite-ensemble marginal and
changes the cross-neuron pairing.
We score both ensembles with the same $M=64$, without a
finite-ensemble Brier correction.
Table~\ref{tab:neural-events} reports the paired comparison.
Retaining dependence reduces the mean Brier score by
approximately $3.1\%$ at 1 NFE and $3.6\%$ at 16 NFE
relative to the shuffled control.
The benefit varies by recording and is close to zero in
recording 025, where test events are rare.
The joint Brier score itself is slightly lower at 1 NFE.
This control measures the forecast value of dependence in
the generated population distribution; it does not identify
synaptic or causal interactions.

\begin{table}[t]
\centering
\small
\setlength{\tabcolsep}{3pt}
\begin{tabular}{rccc}
\toprule
NFE & Joint Brier$\downarrow$
& Shuffled Brier$\downarrow$ & $\Delta\mathrm{Brier}$ \\
\midrule
1
& $0.031276\pm0.000046$
& $0.032283\pm0.000047$
& $0.001007\pm0.000014$ \\
16
& $0.031517\pm0.000213$
& $0.032689\pm0.000218$
& $0.001172\pm0.000017$ \\
\bottomrule
\end{tabular}
\caption{Count Flow Map high-activity event prediction.
$\Delta\mathrm{Brier}$ is the shuffled score minus the joint
score, so positive values favor joint samples.
Entries are mean $\pm$ sample standard deviation across three
seed-level means, each averaging six recordings equally.
Paired differences are computed for each recording and seed
before aggregation.}
\label{tab:neural-events}
\end{table}

% The shuffled control independently permutes draw indices for
% every neuron within each conditioning history.
% This preserves each neuron's finite-ensemble marginal and
% changes the cross-neuron pairing.
% We score both ensembles with the same $M=64$, without a
% finite-ensemble Brier correction.
% At 1 NFE, mean Brier score is $0.031788$ for joint samples
% and $0.032642$ after shuffling.
% The paired increase is $0.000854\pm0.000139$.
% At 16 NFE, the corresponding scores are $0.031626$ and
% $0.032894$, with an increase of $0.001268\pm0.000125$.
% These amount to approximately 2.6\% and 3.9\% reductions
% relative to the shuffled control.
% The benefit varies by recording and is close to zero in
% recording 025, where test events are rare.
% This control measures the forecast value of dependence in
% the generated population distribution.
% It does not identify synaptic or causal interactions.

\paragraph{Calibration.}
To handle discrete ties and finite ensembles, define
the randomized rank
\begin{equation}
R=
\frac{
\#\{m:a^{(m)}<a_y\}
+
U\bigl(\#\{m:a^{(m)}=a_y\}+1\bigr)
}{M+1},
\qquad
U\sim\operatorname{Unif}(0,1).
\end{equation}
For nominal level $c$, we report the fraction of histories
with $R\in[(1-c)/2,(1+c)/2]$.
At nominal 90\%, empirical rank coverage is approximately
68\% at 1 NFE and 72\% at 16 NFE.
Increasing NFE improves coverage, but substantial
undercoverage remains.

\begin{table}[p]
\centering
\small
\setlength{\tabcolsep}{2.5pt}
\begin{tabular}{lrcccc}
\toprule
Method & NFE & Energy score$\downarrow$
& Pop. CRPS$\downarrow$ & RMSE$\downarrow$
& Time (ms)$\downarrow$ \\
\midrule
Count Flow Map & 1
& $0.3922\pm0.0003$ & $0.02248\pm0.00020$
& $0.5506\pm0.0011$ & $2.33\pm0.01$ \\
& 2
& $0.3929\pm0.0002$ & $0.02263\pm0.00011$
& $0.5515\pm0.0008$ & $3.93\pm0.02$ \\
& 4
& $0.3931\pm0.0002$ & $0.02262\pm0.00005$
& $0.5518\pm0.0003$ & $7.10\pm0.03$ \\
& 8
& $0.3931\pm0.0001$ & $0.02257\pm0.00013$
& $0.5520\pm0.0001$ & $13.43\pm0.04$ \\
& 16
& $0.3931\pm0.0001$ & $0.02250\pm0.00013$
& $0.5522\pm0.0001$ & $26.06\pm0.10$ \\
& 32
& $0.3931\pm0.0001$ & $0.02243\pm0.00015$
& $0.5524\pm0.0001$ & $51.21\pm0.17$ \\
& 64
& $0.3931\pm0.0001$ & $0.02246\pm0.00014$
& $0.5524\pm0.0001$ & $101.56\pm0.48$ \\
\midrule
Count-FM (unit jump) & 1
& $0.4229\pm0.0001$ & $0.03508\pm0.00012$
& $0.5851\pm0.0002$ & $1.26\pm0.01$ \\
& 2
& $0.4098\pm0.0002$ & $0.02699\pm0.00016$
& $0.5706\pm0.0003$ & $1.83\pm0.00$ \\
& 4
& $0.4013\pm0.0002$ & $0.02357\pm0.00014$
& $0.5609\pm0.0003$ & $2.96\pm0.00$ \\
& 8
& $0.3967\pm0.0001$ & $0.02246\pm0.00015$
& $0.5558\pm0.0002$ & $5.21\pm0.00$ \\
& 16
& $0.3948\pm0.0001$ & $0.02220\pm0.00016$
& $0.5538\pm0.0002$ & $9.70\pm0.01$ \\
& 32
& $0.3940\pm0.0001$ & $0.02220\pm0.00016$
& $0.5530\pm0.0002$ & $18.59\pm0.04$ \\
& 64
& $0.3936\pm0.0001$ & $0.02220\pm0.00017$
& $0.5527\pm0.0002$ & $36.46\pm0.08$ \\
& 128
& $0.3935\pm0.0001$ & $0.02224\pm0.00016$
& $0.5526\pm0.0002$ & $72.04\pm0.16$ \\
& 256
& $0.3934\pm0.0001$ & $0.02224\pm0.00017$
& $0.5525\pm0.0002$ & $143.36\pm0.94$ \\
\midrule
Count-FM ($\tau$-leap) & 1
& $0.3955\pm0.0003$ & $0.03759\pm0.00029$
& $0.5584\pm0.0010$ & $1.41\pm0.00$ \\
& 2
& $0.3942\pm0.0002$ & $0.02676\pm0.00026$
& $0.5550\pm0.0004$ & $2.11\pm0.01$ \\
& 4
& $0.3937\pm0.0001$ & $0.02333\pm0.00019$
& $0.5535\pm0.0002$ & $3.60\pm0.12$ \\
& 8
& $0.3934\pm0.0001$ & $0.02242\pm0.00015$
& $0.5529\pm0.0001$ & $6.34\pm0.02$ \\
& 16
& $0.3934\pm0.0001$ & $0.02222\pm0.00017$
& $0.5527\pm0.0001$ & $11.96\pm0.02$ \\
& 32
& $0.3934\pm0.0001$ & $0.02222\pm0.00017$
& $0.5526\pm0.0002$ & $23.19\pm0.06$ \\
& 64
& $0.3934\pm0.0001$ & $0.02220\pm0.00017$
& $0.5526\pm0.0002$ & $45.82\pm0.53$ \\
& 128
& $0.3934\pm0.0001$ & $0.02224\pm0.00017$
& $0.5525\pm0.0002$ & $90.32\pm0.28$ \\
& 256
& $0.3934\pm0.0001$ & $0.02224\pm0.00017$
& $0.5525\pm0.0002$ & $179.62\pm0.47$ \\
\bottomrule
\end{tabular}
\caption{Full neural forecasting results within the count-flow family.
Each entry is the mean $\pm$ sample standard deviation of
three seed-level means, each averaging the six recordings
equally. All rows use 4,096 test histories per recording,
64 draws per history, and the same selected checkpoint within
each method, recording, and seed.
The two Count-FM samplers share a rate network.
Time is generation latency per 64-draw forecast under the
common timing protocol. Displayed zero standard deviations
reflect rounding.}
\label{tab:neural-full}
\end{table}

%%%%%%%%%%%%%%%%%%%%%%%%%%%%%%%%%%%%%%%%%%%%%%%%%%%%%%%%%%%%

% \newpage
% \input{checklist.tex}

\end{document}